\documentclass[twoside]{article}

\usepackage[accepted]{aistats2023}

\usepackage[round]{natbib}

\usepackage{hyperref}
\usepackage{url}
\usepackage{booktabs}
\usepackage{amsfonts}
\usepackage{nicefrac}
\usepackage{microtype}
\usepackage{xcolor} 
\usepackage{tikz}
\usetikzlibrary{tikzmark}

\usepackage{stmaryrd}
\usepackage{amssymb}
\usepackage{amsmath}
\usepackage{amsthm}
\usepackage{graphicx}
\usepackage{placeins}
\usepackage{algorithm}
\usepackage{algorithmic}
\usepackage{float}

\DeclareMathOperator*{\argmax}{arg\,max}

\newcommand{\x}{\mathbf{x}}

\newcommand{\y}{\mathbf{y}}

\newcommand{\prob}{\textrm{P}}

\newcommand{\norm}[1]{\left\lVert#1\right\rVert}
\newcommand{\ubar}[1]{\mkern3mu\underline{\mkern-3mu #1\mkern-3mu}\mkern3mu}

\def\XXint#1#2#3{{\setbox0=\hbox{$#1{#2#3}{\int}$ }
		\vcenter{\hbox{$#2#3$ }}\kern-.6\wd0}}

\newtheorem{theorem}{Theorem}
\newtheorem*{theorem*}{Theorem}
\newtheorem{assumption}{Assumption}

\newtheorem{corollary}{Corollary}
\newtheorem*{corollary*}{Corollary}

\newtheorem*{lemma*}{Lemma}
\newtheorem{proposition}{Proposition}
\newtheorem*{proposition*}{Proposition}

\newtheorem*{improvement*}{Improvement}

\newtheorem{definition}{Definition}
\newtheorem{remark}{Remark}

\renewcommand\labelenumi{(\roman{enumi})}
\renewcommand\theenumi\labelenumi

\begin{document}

\twocolumn[

\aistatstitle{Multi-Fidelity Bayesian Optimization with Unreliable Information Sources}

\aistatsauthor{ Petrus Mikkola \And Julien Martinelli \And  Louis Filstroff \And  Samuel Kaski}

\aistatsaddress{Aalto University \And  Aalto University \And ENSAI, CREST \And Aalto University\\ University of Manchester } ]

\begin{abstract}
Bayesian optimization (BO) is a powerful framework for optimizing black-box, expensive-to-evaluate functions. Over the past decade, many algorithms have been proposed to integrate cheaper, lower-fidelity approximations of the objective function into the optimization process, with the goal of converging towards the global optimum at a reduced cost. This task is generally referred to as multi-fidelity Bayesian optimization (MFBO). However, MFBO algorithms can lead to higher optimization costs than their vanilla BO counterparts, especially when the low-fidelity sources are poor approximations of the objective function, therefore defeating their purpose. To address this issue, we propose rMFBO (robust MFBO), a methodology to make any GP-based MFBO scheme robust to the addition of unreliable information sources. rMFBO comes with a theoretical guarantee that its performance can be bound to its vanilla BO analog, with high controllable probability. We demonstrate the effectiveness of the proposed methodology on a number of numerical benchmarks, outperforming earlier MFBO methods on unreliable sources. We expect rMFBO to be particularly useful to reliably include human experts with varying knowledge within BO processes.
\end{abstract}

\section{INTRODUCTION}

Bayesian optimization (BO) has become a popular framework for global optimization of black-box functions, especially when they are expensive to evaluate \citep{Jones1998,brochu2010tutorial}. Such functions have neither known functional form nor derivatives, and conventional optimization techniques such as gradient descent cannot be directly employed. BO rests upon two key elements. First, it constructs a probabilistic surrogate model of the objective function with built-in uncertainty estimates, typically a Gaussian process (GP), based on evaluations of the function. The obtained surrogate is then used to select the next point to evaluate by maximizing of a so-called \emph{acquisition function}, which quantifies the expected utility of evaluating a specific point with the purpose of optimizing the black-box function. Many off-the-shelf acquisition functions achieve this task while balancing exploration and exploitation. Iterating these two steps produces a sequence of designs whose aim is to converge to the global optimum using a limited number of function queries. BO has proven effective for a variety of problems, including hyperparameter optimization \citep{snoek2012practical}, materials science \citep{zhang2020bayesian}, and drug discovery \citep{bombarelli2018lbo,korovina2020chembo}.

In many scenarios, lower-fidelity approximations of the objective function are available at a cheaper query cost. This occurs for instance when the evaluation of the objective function involves a numerical scheme, where computational cost and accuracy can be traded off. Another example is the knowledge of domain experts. Indeed, practitioners may have implicit knowledge of the objective function, for instance, they may be able to point out good candidate regions on the location of the global optimum \citep{hvarfner2022pibo}. Such knowledge may naturally be considered as a low-fidelity version of the true objective function.

The problem of integrating these auxiliary information sources (ISs) to reduce the cost of BO has been tackled in the literature under the name multi-fidelity Bayesian optimization (MFBO) \citep{huang2006sequential,kandasamy2016,zhang2017information,sen2018multi,song2019general,takeno2020multi,li2020multi,moss21gibbon} when the different sources can be ranked by their degree of fidelity; when this is not possible, the problem has been studied as multi-task BO \citep{swersky2013multi}, non-hierarchical multi-fidelity BO \citep{lam2015multifidelity}, or multi-information source BO \citep{poloczek2017multi}. However, as we will empirically demonstrate, state-of-the-art MFBO algorithms can fail when the auxiliary ISs are poor approximations of the primary IS. More precisely, for a fixed budget, these algorithms will lead to a higher regret w.r.t.\ their single-fidelity counterparts (i.e., vanilla BO, which uses the primary IS only), defeating their purpose. For instance, the guarantees of \citet{kandasamy2016} require that the deviation between an auxiliary IS and the primary IS is bounded by a constant known beforehand, which hardly ever holds in practice, for example when working with a human expert, or when experimenting with simulations to find the optimal control parameters for a robotic system~\citep{marco2017virtual}.

Surprisingly enough, this issue has not formally been addressed in the BO literature so far. We fix this gap by introducing rMFBO, a methodology to make any GP-based MFBO algorithm \textit{robust} to the addition of unreliable information sources. More precisely, rMFBO comes with a theoretical guarantee that its performance can be bound to its vanilla BO analog, with high controllable probability. To the best of our knowledge, rMFBO is the first MFBO scheme providing such performance guarantees. We then proceed to demonstrate the effectiveness of the proposed methodology on various numerical settings using different MFBO algorithms of the literature. Through its building block nature, rMFBO paves the way towards a more systematic usage of auxiliary ISs independently of their degree of fidelity, allowing human experts to join the optimization process in a reliable manner.

\section{PRELIMINARIES}

\subsection*{Gaussian process regression}

We begin by introducing the notation for single-output Gaussian process regression, the probabilistic surrogate upon which BO rests. Consider a dataset $\mathcal{D} = \{(\x_1,y_1),\dots,(\x_n,y_n)\}$ with $(\x_i, y_i) \in \mathbb{R}^d\times \mathbb{R}$, for which we want to learn a model of the form $y_i = f(\x_i) + \epsilon$, with $\epsilon \sim \mathcal{N}(0, \sigma^2_{\text{noise}})$ for all $i$. We may place a (zero-mean) GP prior on $f$:
\begin{equation}
    f(\x) \sim \mathcal{GP}(0, k(\x,\x')).
\end{equation}
This defines a distribution over functions $f$ whose mean is $\mathbb{E}[f(\x)] = 0$ and covariance $\text{cov}[f(\x),f(\x')] = k(\x,\x')$. Here, $k$ is a kernel function measuring the similarity between inputs. Consequently, for any finite-dimensional collection of inputs $(\x_1, \dots, \x_n)$, the function values $\boldsymbol{f} = (f(\x_1),\dots,f(\x_n))^{\text{T}}\in \mathbb{R}^n$ follow a multivariate normal distribution $\boldsymbol{f}\sim \mathcal{N}(\boldsymbol{0}, K)$, where $K \in \mathbb{R}^{n\times n} = (k(\x_i, \x_j))_{1\le i,j \le n}$ is the kernel matrix.

Given $\mathcal{D}$, the posterior predictive distribution $p(f(\x) \mid \mathcal{D})$ is Gaussian for all $\x$ with mean $\mu(\x)$ and variance $\sigma^2(\x)$, such that
\begin{align*}
\mu(\x) &= \boldsymbol{k}_{\x} (K+\sigma^2_{\text{noise}}I)^{-1}\mathbf{y}, \\
\sigma^2(\x) &=k(\x, \x) - \boldsymbol{k}_{\x} (K+\sigma^2_{\text{noise}}I)^{-1}\boldsymbol{k}_{\x},
\label{eq:gp_mean_cov}
\end{align*}
where $\mathbf{y} = [y_1,\dots,y_n] \in \mathbb{R}^n$ and $\boldsymbol{k}_{\x} = [k(\x, \x_1),\cdots, k(\x, \x_n)]^{\text{T}} \in \mathbb{R}^n$.

\subsection*{Multi-output Gaussian process regression}

GPs can be extended to multi-output Gaussian processes (MOGP), modeling any collection of $m$-sized vector-valued outputs $(\y_1,...,\y_n)$ based on inputs $(\x_1,...,\x_n)$ as a multivariate normal distribution. One way to achieve this extension is through the addition of a $(d+1)^{\text{th}}$ dimension to the input space, representing the output index $1\le l \le m$. This enables treating a MOGP as a single-output GP acting on the augmented space $\mathbb{R}^{d+1}$ through the kernel $k((\x,\ell),(\x',\ell'))$. The latter can, for instance, take the separable form $k((\x,\ell),(\x',\ell')) = k_{\text{input}}(\x,\x') \times k_{\text{IS}}(\ell,\ell')$. In particular, this setting allows for the use of the readily-available analytical formulae for the posterior mean and variance of single-output GPs.

\subsection*{Problem setup}

We consider the problem of optimizing a black-box
function $f^{(m)} : \mathcal{X} \rightarrow \mathbb{R}$, where $\mathcal{X}$ is a subset of $\mathbb{R}^d$, i.e.\, solving
\begin{equation}\label{BO_task}
\argmax_{\x \in \mathcal{X}} f^{(m)}(\x).
\end{equation}
In addition to $f^{(m)}$ (the \textit{primary} IS), we may also query $m-1$ other auxiliary functions (\textit{auxiliary} ISs), $f^{(\ell)} : \mathcal{X} \rightarrow \mathbb{R}$, where $\ell \in \llbracket m - 1 \rrbracket$ denotes the IS index. The cost of evaluating $f^{(\ell)}(\x)$ is $\lambda_{\ell}$ for any $\x \in \mathcal{X}$. We assume that $\lambda_{\ell} < \lambda_{m}$ for any auxiliary IS $\ell \in \llbracket m - 1 \rrbracket$. The objective is to solve \eqref{BO_task} within the budget $\Lambda > 0$.

\subsection*{Bayesian optimization}

At each round $t$, an input-IS pair $(\x,\ell) \in \mathcal{X} \times \llbracket m \rrbracket$ is selected by maximizing the acquisition function $\alpha$, which depends on the GP surrogate model on $f$ given all the data acquired up until round $t-1$:
\begin{equation}
\x_t, \ell_t = \argmax_{(\x,\ell) \in \mathcal{X} \times \llbracket m \rrbracket} \alpha(\x,\ell).
\label{eq-mfbo}
\end{equation}
Querying for $f^{(\ell)}(\x)$ returns a noisy observation $y_{\x}^{(\ell)} = f^{(\ell)}(\x) + \epsilon$, with i.i.d. noise $\epsilon \sim \mathcal{N}(0,\sigma_{noise}^2)$. We refer to the sequence of queries $\{\x_{t}\}_{t=1}^T$ returned by a BO algorithm as an \textit{acquisition trajectory}.

Note that vanilla BO (i.e., BO with the primary IS only) amounts to using the acquisition function $\x \mapsto \alpha(\x,m)$, which will be referred to as single-fidelity BO (SFBO) from now on.

Recall that we wish to optimize $f^{(m)}$ within the budget $\Lambda$. In this scenario, the performance metric of interest is the regret of the algorithm, whose definition is recalled below.

\begin{definition}[BO regret]
The regret of the BO algorithm that spends $\Lambda_{\text{cost}}$ and returns the final choice $\x_{\text{choice}}$, is defined by
\begin{align*}
    R(\Lambda_{\text{cost}},\x_{\text{choice}}) :=
    \begin{cases}
    f^* - f^{(m)}(\x_{\text{choice}})~~\text{if}~\Lambda_{\text{cost}} \leq \Lambda, \\
    \infty \qquad\qquad\qquad\qquad \text{otherwise}
    \end{cases}
\end{align*}
where $f^* = \max_{\x \in \mathcal{X}}f^{(m)}$ is the global maximum of the primary IS.
\end{definition}

\begin{definition}[Number of queries]\label{n_queries}
Let $T := \lfloor \Lambda / \lambda_m \rfloor$ be the available number of primary IS queries. Let $T^{(\ell)}$ be the random variable describing the number of $\ell^{th}$-IS queries spent by the MFBO algorithm.
\end{definition}

There are two popular choices for $\x_{\text{choice}}$. First, the \textit{Bayes-optimal choice}
\begin{equation}
\x_{\text{choice}} = \argmax_{\x \in \mathcal{X}}\mu(\x,m), \notag 
\end{equation}
where $\mu(\x,m)$ is the posterior mean of the GP model given all the data up to the final query $T_{\text{last}} = \sum_{\ell = 1}^m T^{(\ell)}$. The regret in this case is called the \textit{inference regret}. Second, the \textit{simple choice}
\begin{equation}
\x_{\text{choice}} = \argmax_{t \in \llbracket T^{(m)} \rrbracket}f^{(m)}(\x_t), \notag
\end{equation}
where $(\x_1,...,\x_{T^{(m)}})$ is the primary IS acquisition trajectory returned by the MFBO algorithm. The regret in that case is called the \textit{simple regret}.

Lastly, we provide an informal definition for an unreliable IS. To the best of our knowledge, no formal definition has been proposed in the literature so far, in view of the non-trivial character of the task. Broadly speaking, an auxiliary IS $\ell$ is said to be \textit{unreliable}, if querying $f^{(\ell)}(\x)$ does not lead to a decreasing inference regret when averaged over all sequences of queries $\x \in \mathcal{X}$ and all datasets $\mathcal{D}$ of any size. As such, the relevance on an IS then also depends on the MOGP model chosen.

\section{RELATED WORK}

As discussed in the introduction, many different multi-fidelity extensions of Bayesian optimization have been proposed in the literature; we refer the interested reader to \citet[Section 5]{takeno2020multi} for a review. The closest to our work are methods that do not assume a hierarchy between the sources (e.g., when the degree of fidelity cannot be assessed in advance), as by \citet{lam2015multifidelity}, where the focus lies in designing a GP model that takes into account the non-hierarchical nature of the sources. The multi-fidelity kernel introduced by~\citet{poloczek2017multi} (see Supplementary \ref{supp_mf-kernels}) is one example of such a design.

Surprisingly, the problem of the potential performance degradation of MFBO algorithms has been largely ignored in the literature, with the exception of \citet{kandasamy2016}, who noticed that their multi-fidelity method performed poorly compared to all single-fidelity variants in one experiment \citep[Supplementary D.3]{kandasamy2016}.

Lastly, we mention that robustness has been studied for vanilla BO in the context of (sometimes adversarially) noisy inputs or outputs \citep{martinez2018practical,bogunovic2018adversarially,frohlich2020noisy,kirschner2021bias}. This notion of robustness is fundamentally different from ours, indeed, we wish to provide guarantees that the addition of an auxiliary IS (or several) will not lead to worse performance w.r.t.\ vanilla BO.

\section{PITFALLS OF MFBO METHODS}\label{pitfalls}

\begin{figure*}[t]
	\begin{center}
		\includegraphics[width=\textwidth]{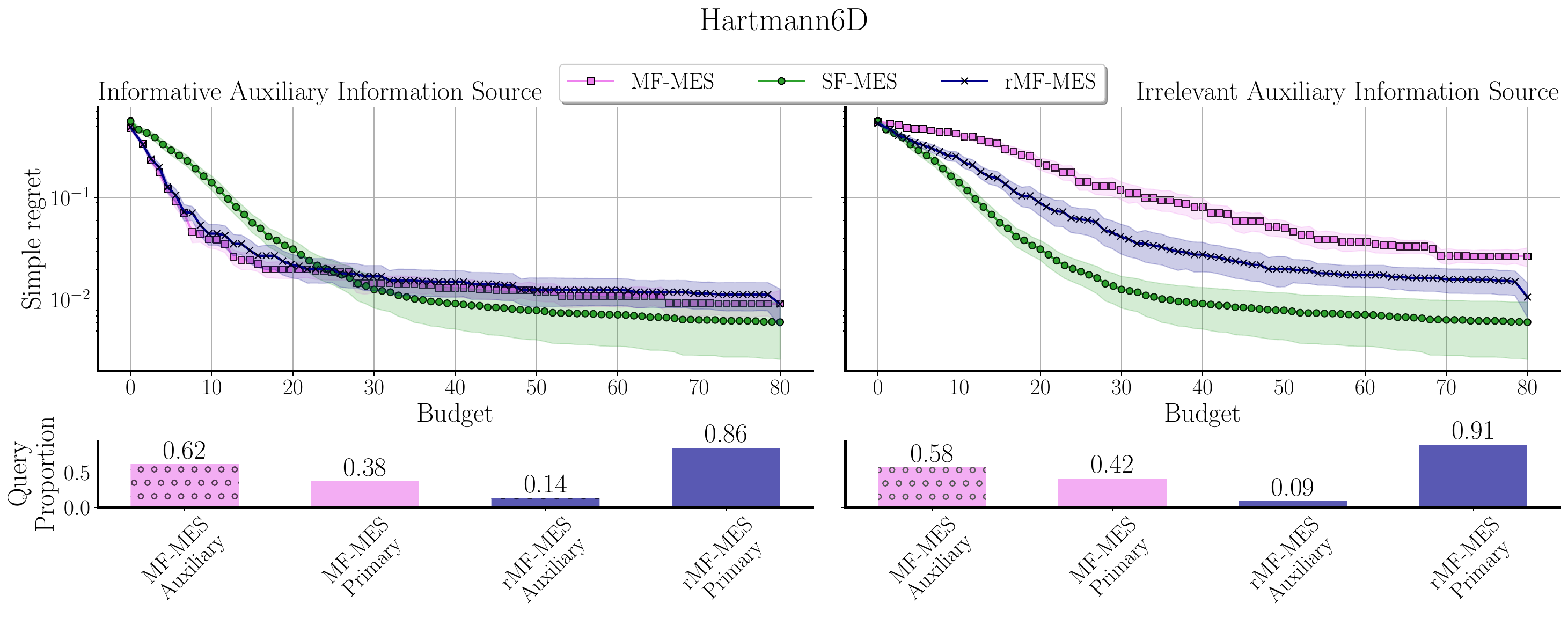}
	\end{center}
	\caption{Simple regret as a function of budget spent in two multi-fidelty problems, averaged over 100 repetitions. The informative auxiliary IS helps reduce the cost of BO (left panel: MF-MES in purple reduces regret faster than SF-MES in green), whereas the irrelevant IS catastrophically disrupts performance (right panel: MF-MES does not reach the low regret of SF-MES at all). In both settings, the primary IS cost is set to 1 and the auxiliary IS cost to 0.2. From relevant to irrelevant IS, the proportion of auxiliary IS queries remains high for MF-MES, while rMF-MES is more consistent (lower panel).} \label{mockup_fig}
\end{figure*}

We now demonstrate, on a simple example, the influence of the auxiliary IS quality on the performance of MFBO algorithms.
Let us consider the Hartmann6D function as the objective (i.e., the primary IS). We examine two scenarios: in the first one, the auxiliary IS is informative, consisting of a biased version of the primary IS, with a degree of fidelity $l=0.2$.
In the second scenario, the auxiliary IS is taken to be the 6-dimensional Rosenbrock function, an irrelevant source for this problem. Analytical forms for these examples can be found in Supplementary~\ref{sec:testfunctions}.
We evaluate the multi-fidelity maximum-entropy search (MF-MES) method from \citet{takeno2020multi} on these two scenarios, as well as its single-fidelity counterpart (SF-MES), and our proposed algorithm, rMFBO, built on top of these methods (rMF-MES).
In both cases, the cost of the primary IS is set to 1, and cost of the auxiliary IS to 0.2. The simple regret of the three algorithms is displayed in Figure~\ref{mockup_fig}.

When the auxiliary IS is informative (left panel), MF-MES converges faster than SF-MES. This is the expected behavior from MFBO algorithms: they use cheap IS queries in the beginning to clear out unpromising regions of the space at a low cost, which eventually speeds up convergence. However, when the auxiliary IS is irrelevant (right panel), there is a clear gap between MF-MES and SF-MES, even in the long run. This demonstrates the inability of MF-MES to deal with an irrelevant IS. In that scenario, we hypothesize that the budget is wasted on uninformative queries, and thus too many rounds are spent on learning that the sources are not correlated (Figure~\ref{mockup_fig}, right bar plot), leading to a sub-optimal data acquisition trajectory.

There is therefore a need for a robust method for such a scenario. This is what the proposed rMF-MES, formally introduced in the next section, achieves, by taking the best of both worlds: sticking close to the single fidelity track in case of an irrelevant IS, while using informative lower-fidelity queries to accelerate convergence.

\section{ROBUST MFBO ALGORITHM}

In this section, we introduce rMFBO (robust MFBO), a methodology to make any GP-based MFBO scheme robust to the addition of unreliable ISs. The key idea is to control the quality of the acquisitions to prevent the MFBO algorithm from behaving as described at the end of Section \ref{pitfalls}.

At round $t$, based on the acquisition function $\alpha$, MFBO proposes the query $(\x^{\text{MF}}_t, \ell_t)$ according to Eq.~\eqref{eq-mfbo}. The question is to decide whether to execute this query, or to go with a more conservative query from the primary IS. Indeed, we wish to curb the potential performance deterioration w.r.t.\ vanilla BO. To do so, we introduce a concurrent pseudo-SFBO algorithm, which constructs a GP surrogate based on data from the primary IS only, and so-called \emph{pseudo-observations}, introduced later on. The pseudo-SFBO uses the acquisition function $\x \mapsto \alpha(\x,m)$ and a separate single-output GP, yielding the query $(\x^{\text{pSF}}_t,m)$.

Let us denote the predictive mean and standard deviation of the MOGP model (used by the MFBO algorithm) by $\mu_{\text{MF}}$ and $\sigma_{\text{MF}}$, and those of the GP model (used by the pseudo-SFBO algorithm) by $\mu_{\text{SF}}$ and $\sigma_{\text{SF}}$. In a nutshell, the proposed rMFBO follows the conservative query $\x^{\text{pSF}}_t$, unless the predictive variance of the MOGP model at $\x^{\text{pSF}}_t$ is small enough:
\begin{equation}\label{cond_var}
\sigma_{\text{MF}}(\x^{\text{pSF}}_t,m) \leq c_1,
\end{equation}
where $c_1>0$ is a user-specified parameter. The pseudo-SFBO learns from all the samples, even when the MFBO candidate $\x^{\text{MF}}_t$ is queried, by adding the pseudo-obervation $\mu_{\text{MF}}(\x^{\text{pSF}}_t,m)$.

It can easily be seen that if $c_1 \rightarrow 0$, the described algorithm becomes the SFBO algorithm, since the MFBO proposals would be always ignored. Our main result, formally discussed in Section \ref{theoretical_results}, is that we are able to derive a lower bound on the regret difference between robust MFBO and SFBO as a function of $c_1 > 0$.

While condition \eqref{cond_var} ensures that we can achieve similar performance as SFBO when auxiliary IS is irrelevant, we also want to reap the benefits of the multi-fidelity approach when the auxiliary IS is relevant. To that end, we introduce a measure of IS relevance, $s$, and add this second condition for the acceptance of the MF query:
\begin{equation}\label{cond_relevance}
s(\x^{\text{MF}}_t,\ell_t) \geq c_2,
\end{equation}
with $c_2 > 0$ a user-specified parameter. We want to draw attention to the fact that condition~\eqref{cond_var} operates over the SFBO proposal while condition~\eqref{cond_relevance} acts on the MFBO proposal, possibly based on an auxiliary IS.
Condition \eqref{cond_relevance} makes rMFBO revert more often to primary IS, and makes pseudo-observations more accurate overall. This allows the algorithm to consider less conservative values for $c_1$, opening the door for the exploitation of auxiliary ISs. In this paper, we use a cost-adjusted information gain \citep{takeno2020multi},
\begin{equation}
    s(\x,\ell) = \frac{\text{I}(f^{(\ell)}(\x), f_* \mid \mathcal{D}^{\text{MF}})}{\lambda_{\ell}} ,
\end{equation}
where $\text{I}$ is the mutual information between the observation $f^{(\ell)}(\x)$ and the maximal value of $f^{(m)}, f_* := \max_{\x \in \mathcal{X}} f^{(m)}(\x)$; in other words the information gain on $f_*$ brought by the observation $f(\x,\ell)$.

The whole procedure is summarized in Algorithm \ref{mfbo-algo}, and an extended version is discussed in Supplementary \ref{sec:fullversion}. Note that lines 24-26 guarantee that if the maximizer is one of the unobserved pseudo-points, it is converted into an actual observation, a requirement in the proof of Theorem \ref{no_harm_theorem}.

\begin{algorithm}[t]
\caption{Robust MFBO algorithm}
\label{mfbo-algo}
    \begin{algorithmic}[1]
        \STATE \textbf{Input}: Budget $\Lambda$, costs $(\lambda_1,...,\lambda_m)$, acquisition function $\alpha$, hyperparameters $c_1$ and $c_2$, relevance measure $s$
        \STATE Initialize $\mathcal{D}^{\text{pSF}},\mathcal{D}^{\text{MF}}$
        \STATE Perform Bayesian updates $\mu_{\text{SF}},\sigma_{\text{SF}},\mu_{\text{MF}},\sigma_{\text{MF}}$
        \STATE $t \gets 1$
        \WHILE {$\lfloor \Lambda / \lambda_m \rfloor \geq 2 \lambda_m$}
            \STATE $\x^{\text{pSF}}_t \gets \argmax_{\x} \alpha(\x,m \mid  \mu_{\text{SF}},\sigma_{\text{SF}})$
            \STATE $(\x^{\text{MF}}_t,\ell_t) \gets \argmax_{\x,\ell} \alpha(\x,\ell  \mid  \mu_{\text{MF}},\sigma_{\text{MF}})$
            \IF{$\sigma_{\text{MF}}(\x^{\text{pSF}}_t,m) \leq c_1$ \AND $s(\x^{\text{MF}}_t,\ell_t) \geq c_2$}
                \STATE $y_t \gets f(\x^{\text{MF}}_t,\ell_t)$
                \STATE $\mathcal{D}^{\text{MF}} \gets \mathcal{D}^{\text{MF}} \cup \{((\x^{\text{MF}}_t,\ell_t),y_t)\}$
                \STATE Perform Bayesian updates $\mu_{\text{MF}},\sigma_{\text{MF}}$
                \STATE $y_t \gets \mu_{\text{MF}}(\x^{\text{pSF}}_t,m)$  \ \# pseudo-observation
                \STATE $\mathcal{D}^{\text{pSF}} \gets \mathcal{D}^{\text{pSF}} \cup \{(\x^{\text{pSF}}_t,y_t)\}$\
                \STATE $\Lambda \gets \Lambda - \lambda_{\ell_t}$
            \ELSE
                \STATE $y_t \gets f(\x^{\text{pSF}}_t,m)$
                \STATE $\mathcal{D}^{\text{pSF}} \gets \mathcal{D}^{\text{pSF}} \cup \{(\x^{\text{pSF}}_t,y_t)\}$
                \STATE $\mathcal{D}^{\text{MF}} \gets \mathcal{D}^{\text{MF}} \cup \{((\x^{\text{pSF}}_t,m),y_t)\}$
                \STATE $\Lambda \gets \Lambda - \lambda_m$
            \ENDIF
            \STATE Perform Bayesian updates $\mu_{\text{SF}},\sigma_{\text{SF}},\mu_{\text{MF}},\sigma_{\text{MF}}$
            \STATE $t \gets t + 1$
        \ENDWHILE
        \STATE $S \gets \left\{ \x \in \mathcal{X}\ \mid\ \sigma_{\text{MF}}(\x,m) \leq c_1\right\}$
        \STATE $\x^{\text{pSF}}_t \gets \argmax_{\x \in S} \mu_{\text{MF}}(\x,m)$
        \STATE $y_t \gets f(\x^{\text{pSF}}_t,m)$
    \end{algorithmic}
\end{algorithm}

\section{THEORETICAL RESULTS}\label{theoretical_results}

In this section we tie the regret of rMFBO to that of its SFBO counterpart. The derivation holds for any relevance measure $s$.

Let us first define the function $f : \mathcal{X} \times \llbracket m \rrbracket \rightarrow \mathbb{R}$ such that $f(\x,\ell) = f^{(\ell)}(\x)$ for all $(\x,\ell) \in \mathcal{X} \times \llbracket m \rrbracket$. We assume that $\mathcal{X}$ is a convex compact subset of $\mathbb{R}^d$, and we make the following assumptions about $f$:

\begin{assumption}[$f$ is drawn from a MOGP]\label{f_is_GP}
Assume $f$ is a draw from a MOGP with zero-mean and covariance function $\kappa((\x,\ell),(\x',\ell'))$. In other words, $\{f(\x_i,\ell_i)\}_i$ is multivariate normal for any finite set of input-IS pairs $\{(\x_i,\ell_i)\}_i$.
\end{assumption}

\begin{assumption}\label{k_is_known}
$\kappa$ is known.
\end{assumption}

\begin{assumption}\label{k_is_in_C2}
$\kappa$ is at least twice differentiable. 
\end{assumption}

These assumptions are common in the Bayesian optimization literature. For instance, see \citet{srinivas_ucb} and \citet{kandasamy2016}. We also follow these authors in the next assumption.

\begin{assumption}[Bounded derivatives with high probability]\label{smallderivative}
\begin{equation*}
\mathbb{P} \left(\underset{\x\in \mathcal{X}}{\sup} \left\lvert \frac{\partial f}{\partial x_j} \right\rvert > L\right) \leq ae^{-(L/b_j)^2}, \quad \forall~j \in \llbracket d \rrbracket
\end{equation*}
for some constants $a,b_j >0$.
\end{assumption}
Since a function with bounded partial derivatives (with an uniform bound $L$) is Lipschitz continuous (with a Lipschitz constant $\sqrt{d}L$), Assumption \ref{smallderivative} implies by complementing and the union bound that
\begin{equation*}
|f^{(m)}(\x) - f^{(m)}(\x')| \leq \sqrt{d} L \norm{\x - \x'}_2 \quad \forall~\x,\x' \in \mathcal{X},
\end{equation*}
with probability greater than $1-dae^{-(L/b)^2}$, where $b := \max_j b_j$. Further, Assumption \ref{smallderivative} is satisfied for four times differentiable kernels \citep[][Theorem 5]{ghosal2006posterior}.

\iffalse
\begin{remark}
From Theorem 5 in \citet{ghosal2006posterior},
\begin{align}
    a &=  O\left(\sup_{(\x,\x') \in \mathcal{X}^2,j \in \llbracket d \rrbracket} \frac{\partial k}{\partial x_j}(\x,\x')  \right)  \\
    b_j &= \sup_{\x \in \mathcal{X}} \mathbb{V}\left(\frac{\partial f}{\partial x_j}(\x)\right) 
\end{align}
\end{remark}
\fi

The next assumptions relate to the acquisition function.
\begin{assumption}\label{alpha_diff}
For any round $t \in \mathbb{N}$, we assume that the mapping $(\x,\mathcal{D}_t) \mapsto \alpha(\x,m,\mathcal{D}_t)$ is  $C^2$.
\end{assumption}
\begin{assumption}\label{alpha_nonsingular_Hessian}
The Hessian $\nabla_{\x}^2 \alpha(\x,m)$ is a definite matrix at the optimum $\x = \x^*$.
\end{assumption}

Running Algorithm \ref{mfbo-algo} for $T$ rounds returns the trajectory $\{\x^{\text{pSF}}_{t}\}_{t=1}^T$, which consists of primary IS queries and pseudo-primary IS queries. Moreover, we denote the acquisition trajectory returned by the single-fidelity counterpart as $\{\x^{\text{SF}}_{t}\}_{t=1}^T$. Our reasoning is as follows: we first control the closeness of the two acquisition trajectories, then derive a lower bound on the difference of their regret. 

Let us consider the dataset as a $t(d+1)$-dimensional vector $\mathcal{D}_t = (x_1^{(1)},...,x_t^{(d)},y_1,...,y_t)$. Let $\mathbb{D}_t$ be the closed line segment joining two datasets $\mathcal{D}_t^{A}$ and $\mathcal{D}_t^{B}$. We introduce a concept, \textit{the maximum rate of change of the next query with respect to $\mathbb{D}_t$}, defined as the random variable $M_t$,
\begin{align*}
M_t = \max_{\mathcal{D} \in \mathbb{D}_t} \norm{\frac{\partial \x_{t+1}}{\partial \mathcal{D}}\left(\mathcal{D}\right)}_{\text{op}},
\end{align*}
where $\norm{\cdot}_{\text{op}}$ is the operator norm. $M_t$ measures the sensitivity of the next query when moving from a dataset $\mathcal{D}_t^{A}$ to a dataset $\mathcal{D}_t^{B}$. It depends on the smoothness of the objective function $f$, the kernel $k$, and the acquisition function $\alpha$. The detailed formulas Eqs.~\eqref{dxdD}-\eqref{M_t} and the computation details for $M_t$ can be found in Supplementary \ref{supp_proofs} and \ref{supp_computeM}. Consider $\mathcal{D}_t^{A} = \mathcal{D}_t^{\text{pSF}}$ and $\mathcal{D}_t^{B} = \mathcal{D}_t^{\text{SF}}$, and let us denote by $\hat{M}_t$ the largest product of any combination of $M_0,...,M_{t-1}$,
\begin{align}
\hat{M}_t = \max_{S \in 2^{\llbracket t-1 \rrbracket}}\prod_{k \in S}M_k.
\end{align}

\begin{proposition}\label{bound_sequence}
Assume Algorithm \eqref{mfbo-algo} has been run with control parameter 
\begin{align}\label{c_1}
    c_1(\varepsilon,q) = \frac{\varepsilon}{\sqrt{-2\log(1-q)}},
\end{align}
for $\varepsilon > 0$ and $q \in (0,1)$. Then, for all $t \in \llbracket T-1 \rrbracket$,
\begin{align}
\norm{\x^{\text{SF}}_{t} - \x^{\text{pSF}}_{t}}_{\infty} &\leq \varepsilon t\hat{M}_t d^{t}
\end{align}
holds with probability greater than $q\left(1-da\exp(-1/b^2)\right)$.
\end{proposition}

\begin{proof}
The proof for the noiseless scenario ($\sigma_{\text{noise}}=0$) is given in Supplementary \ref{noiseless_scenario}. For a noisy scenario ($\sigma_{\text{noise}}>0$), see the proof in Supplementary \ref{noisy_scenario}. In the noisy scenario, the statement holds with probability greater than $\textrm{erf}\left(\frac{1}{2\sqrt{\sigma_{\text{noise}}}}\right)\textrm{erf}\left(\frac{1}{\sqrt{2\sigma_{\text{noise}}}}\right)q\left(1-da\exp(-1/b^2)\right)$, where $\textrm{erf}$ is the Gauss error function.
\end{proof}

\begin{corollary}\label{instant_regret_bounded}
The instant regret difference for all $t \in \llbracket T-1 \rrbracket$ is bounded; we have
\begin{align}
    \left|f^{(m)}(\x_t^{\text{SF}})-f^{(m)}(\x^{\text{pSF}}_{t})\right| \leq  \varepsilon t\hat{M}_t d^{t+1}
\end{align}
with probability greater than $q\left(1-da\exp(-1/b^2)\right)$.
\end{corollary}

\begin{proof}
By Assumption \ref{smallderivative} and the equivalence of the norms $\norm{\cdot}_2$ and $\norm{\cdot}_{\infty}$ (with a constant $\sqrt{d}$), it holds with probability greater than $1-da\exp(-\frac{1}{b^2})$ that
\begin{equation*}
\left|f^{(m)}(\x)-f^{(m)}(\x')\right| \leq  d \norm{\x - \x'}_{\infty},
\end{equation*}
for all $\x,\x' \in \mathcal{X}$. Corollary \ref{instant_regret_bounded} follows from Proposition \ref{bound_sequence}.
\end{proof}

We can now present our main result, which states that with a conservative control parameter $c_1(\varepsilon, q)$ (small $\varepsilon$ and high $q$), the worst case regret remains close to the regret of the SFBO algorithm, with high probability. This means that including an auxiliary IS in the robust MFBO algorithm will not cause any ``harm'', given conservative control parameters.

\begin{theorem}[“No harm"]\label{no_harm_theorem}
Assume that both algorithms, the robust MFBO (Algorithm \ref{mfbo-algo}) and its SFBO variant, return their simple final choices. Then,
\begin{align}
R(\Lambda + \lambda_m, \x_{\text{choice}}^{\text{rMF}})~\leq~ &R(\Lambda, \x_{\text{choice}}^{\text{SF}}) \\ & + \varepsilon \max \big\{ T\hat{M}_T d^{T+1},2\big\} \notag,
\end{align}
with probability greater than $q\left(1-da\exp(-\frac{1}{b^2})\right)$.
\end{theorem}

\begin{proof}
The proof is given in Supplementary \ref{proof_no_harm_theorem}.
\end{proof}

Theorem \ref{no_harm_theorem} says that the magnitude of a possible regret loss, compared to SFBO, is proportional to $\varepsilon$ with a probability proportional to $q$. For instance, if we tolerate $0.1$ units of regret undershoot with $90\%$ probability, then by Theorem \ref{no_harm_theorem} this is guaranteed with the control parameter value $c_1 = c_1(0.1,0.9) \approx 0.05$ for early BO rounds.

The regret difference bound in Theorem \ref{no_harm_theorem} corresponds to the worst-case scenario where the multi-fidelity queries $(\x^{\text{MF}}_t, \ell_t)$ are always accepted, and hence the bound is not tight. The bound is practically useful only in the first rounds due to the exponential dependence on $T$, however, in practice the bound is tighter than stated in Theorem \ref{no_harm_theorem}, because the acceptance probability is never one (due to $c_2 > 0$).

\section{EXPERIMENTAL RESULTS}\label{experiments}

We evaluate rMFBO on a benchmark of synthetic functions widely used in multi-fidelity studies. Our proposed method is used to make three state-of-the-art MFBO algorithms robust: Maximum Entropy Search (MF-MES)~\citep{takeno2020multi}, General-purpose Information-Based Bayesian OptimisatioN (MF-GIBBON)~\citep{moss21gibbon} and Knowledge Gradient (MF-KG)~\citep{poloczek2017multi}, the latter being benchmarked only on low-dimensional problems. All experiments are run within the BoTorch framework~\citep{balandat2020botorch}. The probabilistic surrogate model uses the downsampling kernel from \citet{wu2019practical}: $k((\x,\ell),(\x',\ell')) = k_{\text{input}}(\x, \x') \times k_{\text{IS}}(\ell,\ell')$, where $k_{\text{input}}(\cdot,\cdot)$ is the RBF kernel, and $k_{\text{IS}}(\ell,\ell') := c + (1-\ell)^{1+\delta}(1-\ell')^{1+\delta}$. Here, $l \in [0,1]$ represents the degree of fidelity of the IS, $l=1$ corresponding to the target function, often denoted $m$.
The hyperparameters $c$ and $\delta$ are estimated by maximum marginal likelihood, similarly as those of $k_{\text{input}}$. Results using alternative probabilistic surrogates can be found in Supplementary~\ref{supp_mf-kernels}.  Each test function is rescaled in $[0, 1]$. Analytical forms can be found in Supplementary~\ref{sec:testfunctions}, together with 2D plots when applicable. For all experiments, the initial dataset consists of $5d$ evaluations of the primary IS and $4d$ evaluations of each auxiliary IS. For the remainder of the section, except in the dedicated ablation study, we set rMFBO hyperparameters to $c_1=c_2=0.1$.
The latter corresponds to about $15\%$ of the maximum information gain. For more information on choosing $c_2$, see Supplementary \ref{supp_c2}. For each experiment, we report the average and standard deviation of the simple regret computed over 100 repetitions with different initializations. Results displaying the inference regret instead can be found in Figure~\ref{fig:infregret} in the Supplementary. Source code reproducing the experiments is available at \url{https://github.com/AaltoPML/rMFBO}.

\subsection{Synthetic functions with one auxiliary IS}

The goal is to maximize a target function using noisy evaluations of the objective (primary IS), and an auxiliary IS. This is exactly the setting of the introductory example of Section \ref{pitfalls}, where we discussed results obtained with an informative IS case, and an irrelevant IS case. In this subsection, we provide results with two additional objective functions and types of auxiliary ISs.

\paragraph{Negated Exponential Currin 2D:}

We reproduce the experiment performed by~\citet{kandasamy2016} and consider the Exponential Currin function as the primary IS while the auxiliary IS is the negated objective function itself, with $\lambda_l=0.1, \lambda_m=1$. The proposed rMFBO is able to find the global optimum using a lower budget than its MFBO counterpart (Figure~\ref{fig:bigmatrix}, first row), even though performances are quite similar for the Knowledge Gradient acquisition function.

\paragraph{Sinus-perturbed Rosenbrock 2D:}

Next, we examine the Rosenbrock 2D function as the target. In this experiment, the auxiliary IS is equal to the objective corrupted with a sinusoidal signal whose magnitude is the target function mean. The costs are $\lambda_l = 0.2, \lambda_m=1$. rMFBO consistently improves over MFBO and leads to a simple regret on par with SFBO (Figure~\ref{fig:bigmatrix}, second row). It is worth noticing that a similar experiment was performed by~\citet{poloczek2017multi}, however using a sinusoidal signal with a magnitude 200 times inferior to the mean of the objective function. Applying a more realistic perturbation still produces an informative IS, while illustrating a slightly increased robustness brought by rMF-KG with respect to MF-KG.

\begin{figure*}[t]
\begin{center}
		\includegraphics[width=\textwidth]{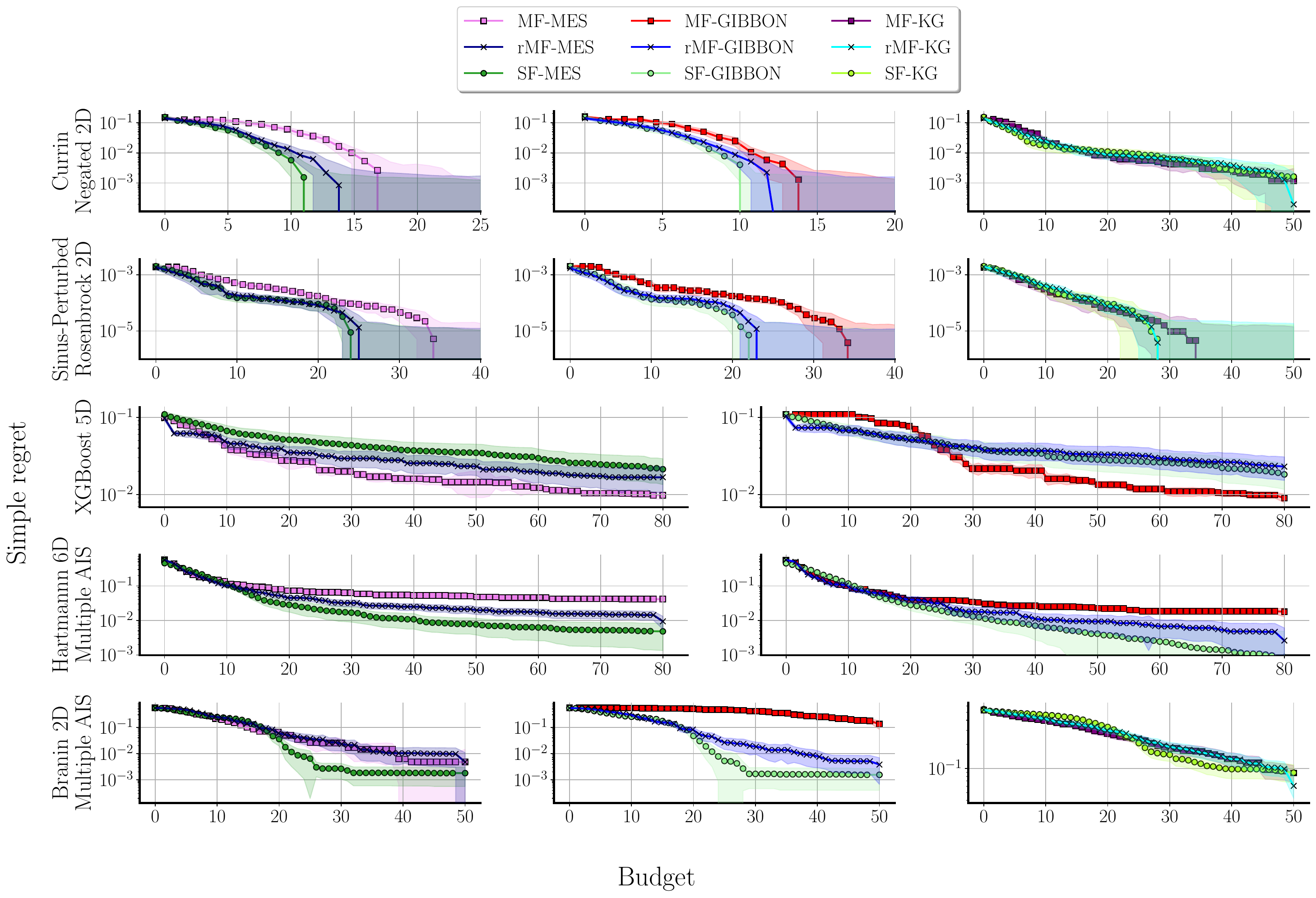}
	\end{center}
	\caption{Simple regret depicted over budget spent in five multi-fidelity problems, averaged over 100 repetitions. The first three problems have one auxiliary IS and the last two have three. Three BO acquisition functions (MES, GIBBON, KG) are tested with their multi- and single-fidelity variants. The proposed \textit{robust} multi-fidelity method is denoted by the letter `r', e.g.\ `rMF-MES'.}
\label{fig:bigmatrix}
\end{figure*}

\subsection{XGBoost hyperparameter tuning}

We now assess the performance of rMFBO on a real-world hyperparameter tuning example.
To that end, we follow the experiment introduced by~\citet{li2020multi} and train an XGBoost model~\citep{chen2016xgb} to predict a quantitative measure of the diabetes progression\footnote{\scriptsize{\url{https://archive.ics.uci.edu/ml/datasets/diabetes}}}.
The dataset includes 442 examples, two-thirds are used for training and the remaining fraction for evaluation. 
We employ the implementation from the scikit-learn library~\citep{scikit-learn} and optimize 5 continuous hyperparameters described in Section~\ref{sec:testfunctions}.
The primary IS trains XGBoost with 100 weak learners trees, while the auxiliary IS uses only 10, with $\lambda_l = 0.1, \lambda_m = 1$.
The optimization starts with 10 random queries at each IS.
We use the normalized root mean square error to evaluate the performance.
rMF-MES is able to take advantage of the auxiliary IS and achieves faster convergence than SF-MES, while staying close to MF-MES (Figure~\ref{fig:bigmatrix}, third row). Note that the auxiliary IS is in this case of extremely good quality and comes at a 10-times cheaper cost (see Figure~\ref{fig:xgbdist} in the Supplementary), thus demonstrating the ability of rMF-MES to use these cheap queries even though the main purpose of the algorithm is to provide robustness against unreliable sources.
On the other hand, rMF-GIBBON displays a behavior closer to SF-GIBBON. While a significant advantage over MF-GIBBON can be observed in the first half of the budget, rMF-GIBBON then gets distanced by its MFBO counterpart. 

\subsection{Synthetic functions with several auxiliary ISs of varying relevance}

Next, we check whether rMFBO is able to distinguish between relevant and irrelevant auxiliary ISs in the presence of multiple auxiliary ISs on two examples.
In the first problem, the 6D Hartmann function is selected as the primary IS, with 3 auxiliary ISs: Hartmann with degrees of fidelity 0.8, 0.1, and the Rosenbrock function, respectively. In the second problem, we wish to optimize the 2D Branin function with 3 auxiliary ISs: Branin with degree of fidelity 0.8, 0.1, and the Ackley function. In both settings, the primary IS can be queried for a cost of 1, and all auxiliary ISs for a cost of 0.2.
Regarding the Hartmann problem, rMFBO provides a consistent decrease of regret compared to MFBO across all methods. Our claim is backed up by the distribution of queries (Figure~\ref{fig:barplot}, top row), which show how rMFBO avoids querying uninformative sources while still taking advantage of the informative one, whereas MFBO queries are scattered across all ISs, delegating around 40\% of them to ISs that are either irrelevant or not worth the incurred cost.
Performance is slightly worse for the Branin problem, where rMFBO significantly improves over MFBO for the GIBBON acquisition function only. With a regret 100 times superior to that of MES and GIBBON methods, the KG method seems to converge very slowly for this example, even though only the primary IS is queried (Figure~\ref{fig:barplot}, lower panel, right). The close results between MF-MES and rMF-MES can be explained by the fact that MF-MES only spends 26\% of its queries on irrelevant ISs while dedicating 50\% of its budget on a cheap and relevant AIS.

\subsection{Ablation study}

Lastly, we investigate the performance of rMFBO w.r.t. several changes such as auxiliary IS cost or hyperparameters values. The benchmarks are only performed using the MES and GIBBON strategies, due to the high computational overhead of the KG method, more than ten times larger than MES and GIBBON \citep{moss21gibbon}. We study the same setting described at the end of Section~\ref{pitfalls}. All discussed figures are in the Supplementary material.

\paragraph{Variation of the auxiliary IS cost $\lambda_l$:}

Figure~\ref{fig:hartmanncost}
shows the evolution of the simple regret as the lower fidelity cost increases ($\lambda_l \in \{0.1, 0.2, 0.5, 0.8\}$). For the informative auxiliary IS case (first two columns), the increase in the cost naturally shifts the behavior of MFBO and rMFBO towards that of SFBO, losing the the acceleration of the convergence in the process. This illustrates the trade-off between informativeness of an auxiliary IS and its query cost.
Interestingly, MFBO methods fail to respect that trade-off in an irrelevant auxiliary IS setting. Indeed, MFBO regret eventually flats out for both MES and GIBBON acquisition strategies at high auxiliary IS query cost. For $\lambda_l = 0.8$, MF-MES spent 82\% of the budget on auxiliary IS queries on average, MF-GIBBON 42\%, over the 20 repetitions. Instead, rMFBO is, as expected, only slightly positively affected by the cost increase, since the irrelevant auxiliary IS becomes less and less relevant, showing again robustness to uninformative IS.

\paragraph{Variation of rMFBO hyperparameters:}
We now investigate how the hyperparameters $c_1$ and $c_2$ affect the performance of rMFBO. $c_1$ constitutes a threshold on the primary IS MF model posterior variance evaluated at SFBO proposal $(\x^{\text{pSF}}_t, m)$, which leads to a primary IS query when exceeded. $c_2$ measures the information gain provided by the MFBO proposal $(\x^{\text{MF}}_t, l)$ for $l \in \llbracket m \rrbracket$, and leads to a primary IS query when not exceeded. We vary $c_1,c_2 \in \{0,0.1,0.2\}$ and display the results in 
Figure~\ref{fig:hartmannablation}.

By construction, setting $c_1$ to $0$ (first two rows) essentially reduces rMFBO to SFBO, no matter the value of $c_2$. As $c_1$ increases and $c_2=0$ (first column, third to sixth rows), rMFBO quickly transitions to the MFBO dynamics regardless of the relevance of the IS. Then, increasing $c_2$ provides a reasonable trade-off between robustness to irrelevant auxiliary IS and exploitation of informative auxiliary IS.

A tentative adaptive strategy to automatically set $c_2$ is presented in Supplementary Section~\ref{sec:adap} and illustrated in Figure~\ref{fig:bigmatrixadap}. Results are comparable with that of Figure~\ref{fig:bigmatrix}, where $c_2 = 0.1$.

\paragraph{Varying the kernel confidence level:}
The downsampling kernel used as joint model here
involves a contribution that depends on the value $\ell$ associated to each IS (Equation~(\ref{eq:downsamp})). While the primary IS $m$ always corresponds to the value $\ell=1$, the value set for the auxiliary IS remains to be selected, and determines whether that IS is perceived as relevant in the optimization of the objective. In Figure~\ref{fig:hartmannconfidence}, for each row, the auxiliary IS value $\ell$ varies in $\{0.1, 0.2, 0.5, 0.8\}$, so as to simulate increased confidence in the auxiliary IS for the MOGP. As expected, gradually increasing the level of confidence in the irrelevant IS case leads to optimization failure for MF-MES and MF-GIBBON (right panel), while rMF-MES and rMF-GIBBON maintain steady performances. In the relevant auxiliary IS case (left panel), there seems to exist an optimal value of $\ell \approx 0.1$ matching the informativeness of the AIS, so that MFBO methods perform on par with rMFBO. Beyond that value, overconfidence leads to degraded performance for MFBO.

\paragraph{Varying the MOGP model:}
Previous experiments were performed using the downsampling kernel as the joint model. We next employ the linear truncated kernel available from the BoTorch library, as well as the MISO kernel introduced in~\citep{poloczek2017multi}. These kernels are described in Section~\ref{supp_mf-kernels} and yield similar performance as the downsampling kernel (Figures~\ref{fig:bigmatrixlt} and~\ref{fig:bigmatrixmiso}), thus demonstrating that rMFBO is robust to the choice of a specific Gaussian process surrogate model.

\section{CONCLUSIONS}

In this paper, we introduced rMFBO, a building block to any MFBO method to make it robust to unreliable information sources, i.e., which do not decrease the regret on average when queried and therefore harmful to the optimization process. In particular, we showed that the regret bound of rMFBO can be tied to that of SFBO, with high probability. Upon extensive experiments, we further demonstrated that the current MFBO methods lack this notion of robustness, and that rMFBO was able to successfully fill this gap, while staying competitive when the auxiliary information sources are relevant.

The proposed rMFBO method relies on two hyperparameters, $c_1$ and $c_2$. While $c_1$ is theoretically grounded, 
$c_2$ was set to a single fixed value which, even though not having deeper theoretic grounding, produced satisfactory results across the wide range of experiments. Its soundness was further empirically assessed through an ablation study. Nevertheless, $c_2$ should adapt to the number of BO rounds, as would be expected from an entropy-based measure. A tentative adaptive approach, which produced satisfactory results, was considered in Supplementary \ref{supp_c2}. Gaining theoretical understanding on how this value should be selected is left for future research.
From a computational perspective, rMFBO keeps track of two acquisition trajectories, which leads to increased computation times, but negligible compared to the evaluation costs encountered in real-world settings. Lastly, the regret bound of rMFBO is practically useful only in the first rounds due to the exponential dependence on the number of BO rounds. Future research should use non-zero rejection probability of the multi-fidelity query proposals to derive a tighter regret bound for later BO rounds.

Any safety-critical MFBO application can benefit from a methodology such as rMFBO, as our algorithm gives guarantees against erroneous or even adversarial information sources.
Finally, rMFBO opens the door to a more systematic inclusion of human experts, with varying knowledge, within BO processes. Typically, these experts would have precise understanding on a specific region of the input domain, but would provide irrelevant feedback elsewhere.
Our algorithm makes it possible to take into account these novel information sources with varying degree of fidelity across the input domain, opening exciting opportunities in Bayesian optimization.

\subsubsection*{Acknowledgements}

Louis Filstroff was with the Department of Computer Science of Aalto University at the time this research was conducted. This work was supported by the Academy of Finland (Flagship programme: Finnish Center for Artificial Intelligence FCAI and decision 341763), EU Horizon 2020 (European Network of AI Excellence Centres ELISE, 951847; HumanE AI Net, 952026), UKRI Turing AI World-Leading Researcher Fellowship (EP/W002973/1). We also acknowledge the computational resources provided by the Aalto Science-IT Project from Computer Science IT.

\bibliographystyle{plainnat}
\bibliography{ref}

\clearpage

\appendix

\thispagestyle{empty}

\onecolumn 

\aistatstitle{Supplementary Material 

Multi-Fidelity Bayesian Optimization with Unreliable Information Sources}

\section{Additional experimental results}

\begin{figure*}[h]
\begin{center}
		\includegraphics[width=\textwidth]{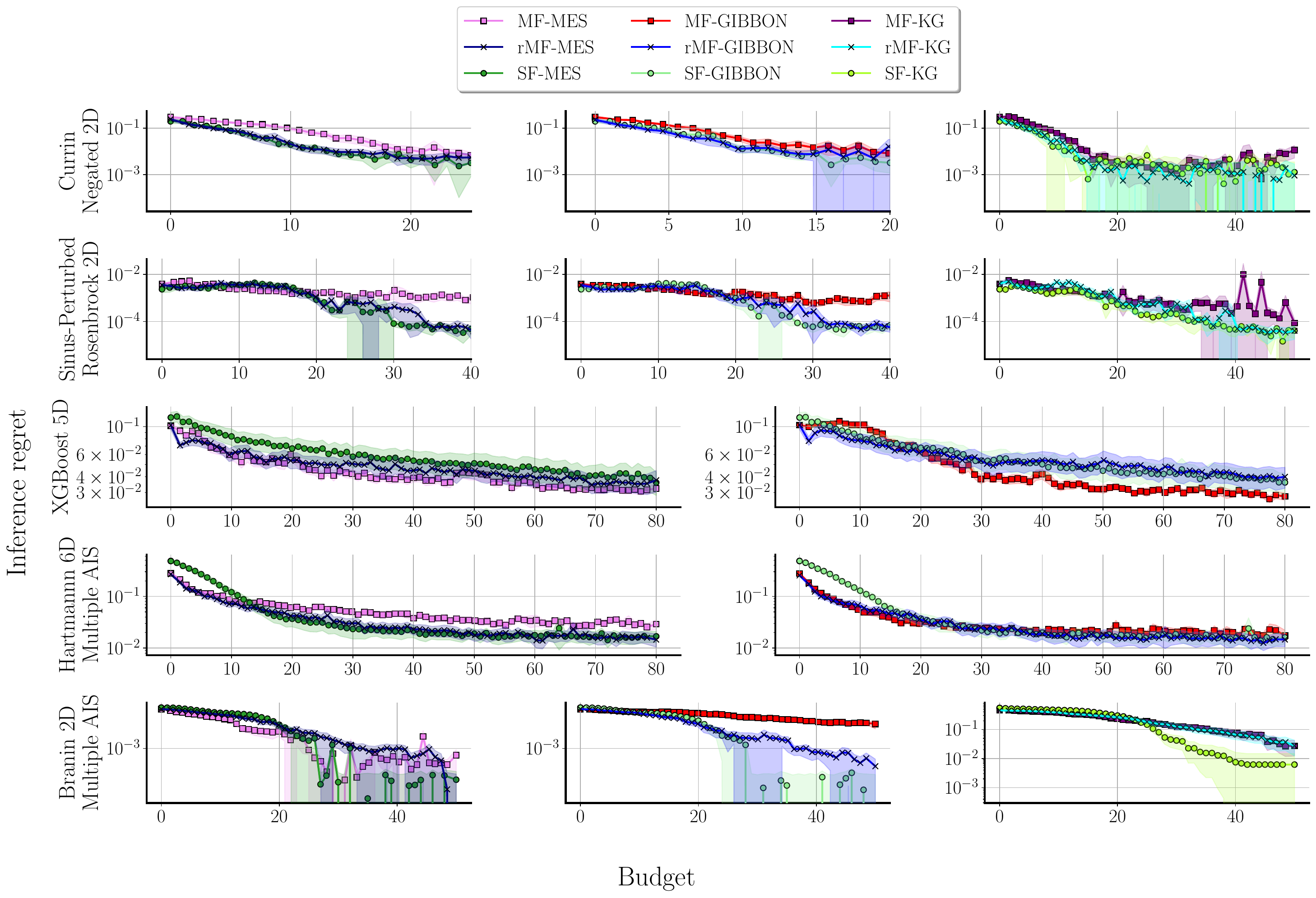}
	\end{center}
	\caption{Inference regret depicted over budget spent in five multi-fidelity problems, averaged over 100 repetitions. The first three problems have one auxiliary IS and the last two have three. Three BO acquisition functions (MES, GIBBON, KG) are tested with their multi- and single-fidelity variants. The proposed \textit{robust} multi-fidelity method is denoted by the letter `r', e.g.\ `rMF-MES'.}
\label{fig:infregret}
\end{figure*}

\begin{figure}[h]
  \centering
	\includegraphics[width=.8\textwidth]{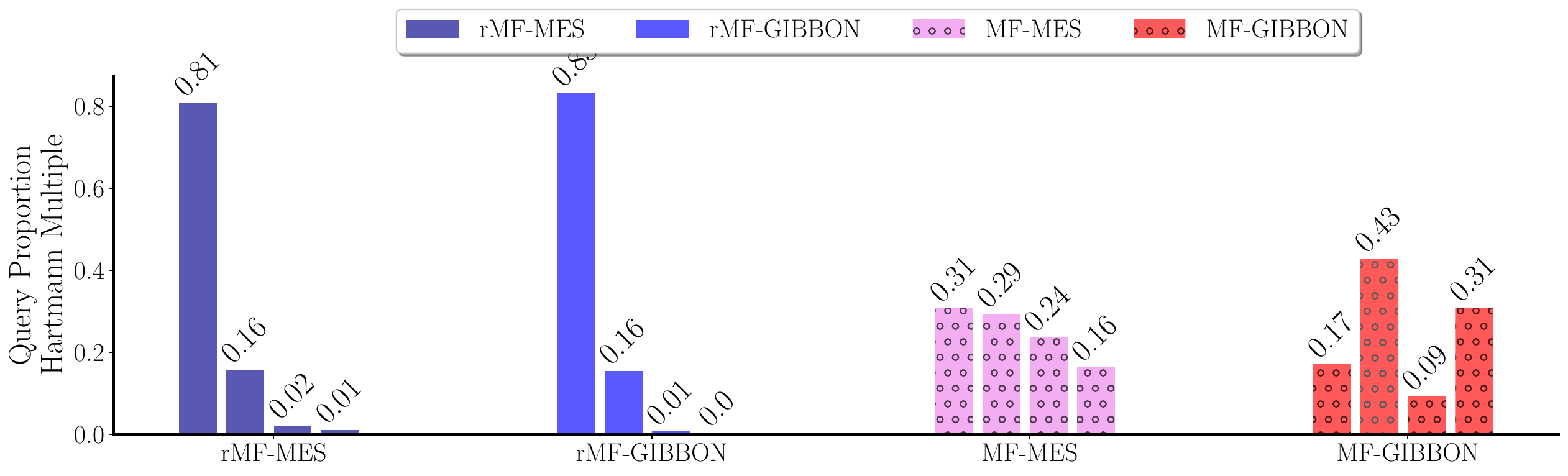}
	\includegraphics[width=.8\textwidth]{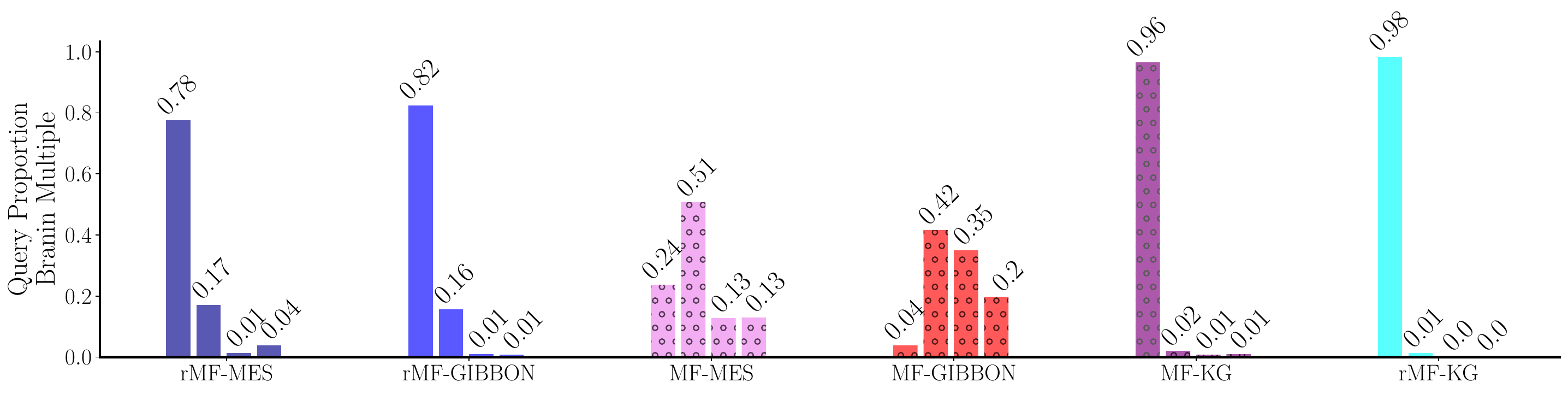}
	\caption{Distribution of IS queries. From left to right, the bars are sorted in the following order: Hartmann (resp.\ Branin) primary IS , AIS with degree of fidelity l = 0.8, AIS with l = 0.1 and finally the Rosenbrock (resp.\ Ackley) function. We used the downsampling kernel.  In the computation of $k_{\text{IS}}(\ell, \ell')$, we used $\ell=1,~\ell'=0.8,~\ell^{\prime\prime}=0.1$ and $\ell^{\prime\prime\prime}=0$.}
\label{fig:barplot}
\end{figure}

\begin{figure}[t]
  \centering
	\includegraphics[width=0.925\linewidth]{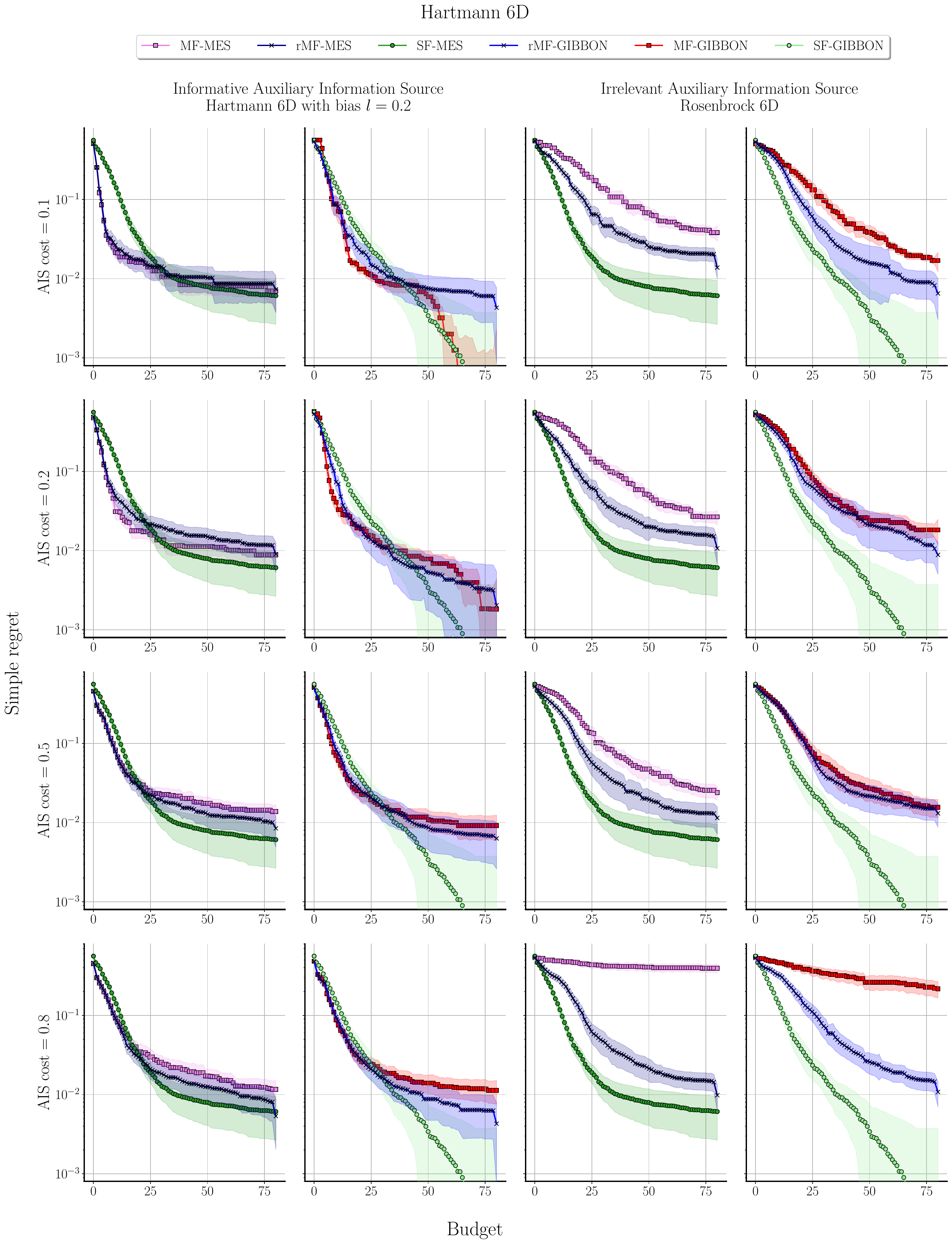}
	\caption{Simple regret depicted over budget spent for the Hartmann6D multi-fidelity problem, averaged over 100 repetitions. For each row, the auxiliary IS cost is varied. We used the downsampling kernel. In the computation of $k_{\text{IS}}(\ell, \ell')$, we used $\ell=1,~\ell'=0.2$ for the primary IS and the auxiliary IS, respectively.}
\label{fig:hartmanncost}
\end{figure}

\begin{figure}[t]
  \centering
	\includegraphics[width=0.925\linewidth]{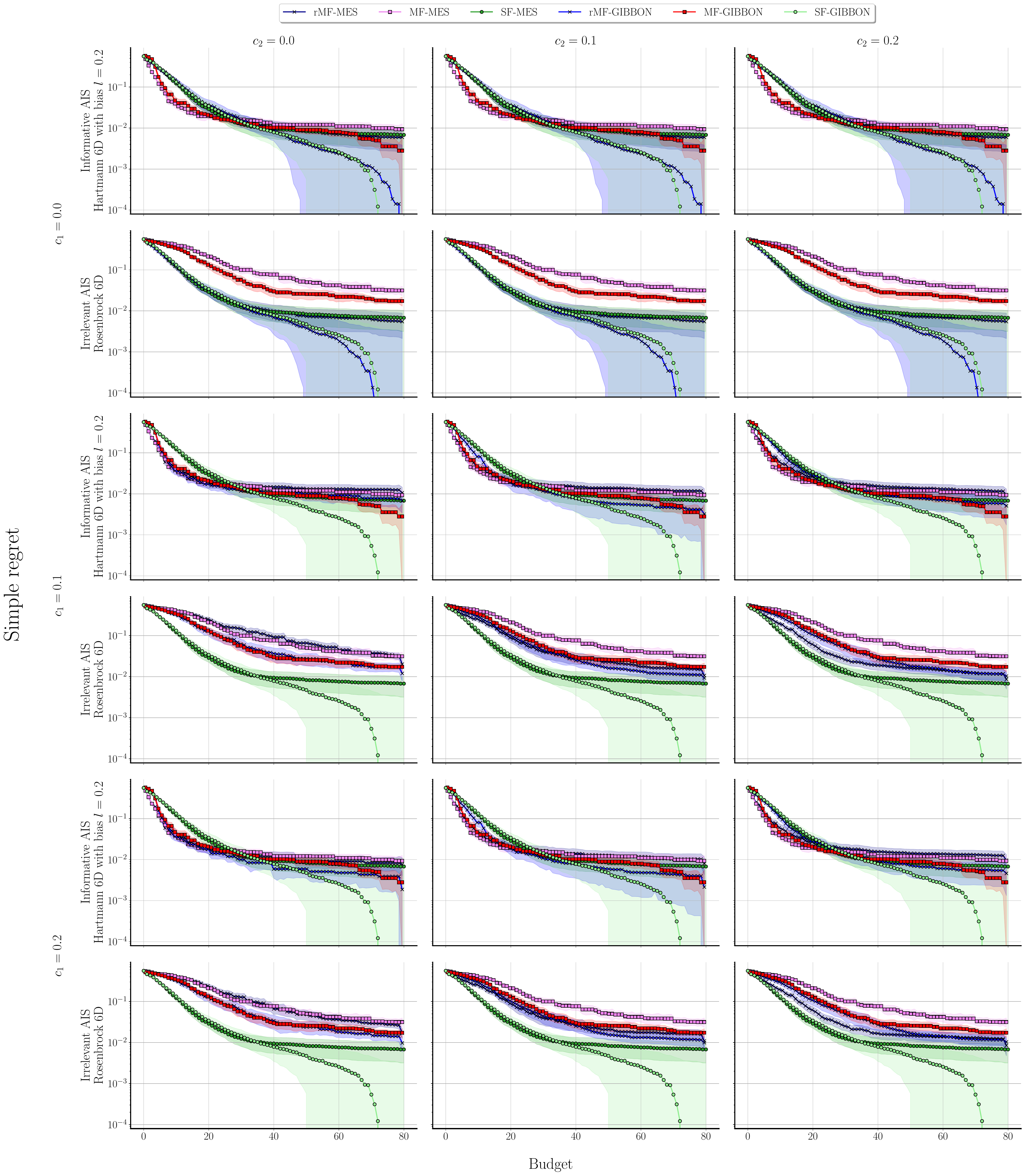}
	\caption{Simple regret depicted over budget spent for the Hartmann6D multi-fidelity problem, averaged over 100 repetitions. At the row level, the hyperparameter $c_1$ is varied, with even rows being the Hartmann/Hartmann0.2 (relevant IS) problem, while odd rows consider the Hartmann/HartmannRosenbrock (irrelevant IS) problem. At the column level, the hyperparameter $c_2$ is varied. We used the downsampling kernel. In the computation of $k_{\text{IS}}(\ell, \ell')$, we used $\ell=1,~\ell'=0.2$ for the primary IS and the auxiliary IS, respectively.}
\label{fig:hartmannablation}
\end{figure}

\clearpage

\section{Proofs}\label{supp_proofs}

\subsection{Noiseless scenario: Proposition \ref{bound_sequence}}\label{noiseless_scenario}

\begin{proof}[Proof of Proposition \ref{bound_sequence}]\label{proof_bound_sequence}

In the proof, we simplify the notations and denote the pseudo SFBO query at round $t$ by $\x_t = \x^{\text{pSF}}_t$, and  $\mathcal{D}_t = \mathcal{D}_t^{\text{pSF}}$ for the pseudo SFBO dataset. The SFBO query at round $t$  is denoted by $\x_t^{\text{SF}}$, as earlier. The acquisition function is treated as a function of an input-IS pair $(\x,\ell)$ and the dataset $\mathcal{D}$.
Let us think the dataset as a $t(d+1)$-dimensional vector $\mathcal{D}_t = (x_1^{(1)},\cdots,x_1^{(d)},\cdots,x_t^{(1)},\cdots,x_t^{(d)},y_1,...,y_t)$. Let us use a shorthand notation $\alpha(\x,m,\mathcal{D}) = \alpha(\x,\mathcal{D})$, and consider the next primary IS query
\begin{equation}\label{sf_alpha_opt}
\x_{t+1} = \argmax_{\x \in \mathcal{X}}\alpha(\x, \mathcal{D}_t)
\end{equation}
which defines an implicit function $\x_{t+1}(\mathcal{D}_t)$ such that $\x_{t+1}$ solves $\nabla_{\x}\alpha(\x, \mathcal{D}_t) = \mathbf{0}_{\mathbb{R}^d}$. By applying the implicit function theorem to the continuously differentiable function $\nabla_{\x}\alpha : \mathcal{X} \times (\mathcal{X}\times \mathbb{R})^{t} \rightarrow \mathbb{R}^d$ (Assumption \ref{alpha_diff}) with invertible Jacobian (Assumption \ref{alpha_nonsingular_Hessian}),
the rate of change of the next query with respect to the dataset can be defined as,
\begin{align}\label{dxdD}
\norm{\frac{\partial \x_{t+1}}{\partial \mathcal{D}_t}\left(\mathcal{D}_t\right)}_{\text{op}} = \norm{\left[\left(\mathbf{H}_{\alpha,\x}(\x_{t+1}(\mathcal{D}_t),\mathcal{D}_t) \right)^{-1}\frac{\partial \nabla_{\x}\alpha}{\partial\mathcal{D}_t^{(k)}}(\x_{t+1}(\mathcal{D}_t),\mathcal{D}_t)\right]_{k=1}^{t(d+1)}}_{\text{op}}
\end{align}
where $\norm{A}_{\text{op}} := \inf \{c>0 : \norm{A\x}_\infty \leq c\norm{\x}_\infty\ \forall \x \in \mathcal{X} \}$ is the operator norm, and $\mathbf{H}_{\alpha,\x}$ denotes the Hessian of $\x \mapsto \alpha(\x,\mathcal{D}_t)$. Specifically, the $(i,j)^{\text{th}}$-element of $\mathbf{H}_{\alpha,\x}$ reads as $\frac{\partial^2 \alpha}{\partial x_i \partial x_j}$. The $i^{\text{th}}$-element of $\frac{\partial\nabla_{\x} \alpha}{\partial\mathcal{D}^{(k)}_t}$ is $\frac{\partial^2 \alpha}{\partial\mathcal{D}^{(k)}_t\partial x_i}$, which denotes the partial derivative w.r.t.\ the first and the second variable of $\alpha(\cdot,\cdot)$.

We now show the proposition by induction.

$1^{\circ}$ Base case $t=1$: The claim follows by the design of Algorithm, since $\x^{\text{SF}}_{1}=\x_{1}$.

$2^{\circ}$ Induction step: Let us assume that $t \in \{1,...,T-2\}$. The outcome vector of the SFBO algorithm is $\y^{SF}_t = (y_1^{\text{SF}},...,y_{t}^{\text{SF}})$. The outcome vector of the pseudo-SFBO algorithm, $\y_t = (y_1,...,y_{t})$, consists of observations $y_{\tau} = f^{(m)}(\x_{\tau})$ and pseudo-observations $y_{\tau} = \mu_{\text{MF}}(\x_{\tau},m)$.

Consider the mapping $\x_{t+1} : (\mathcal{X}\times \mathbb{R})^{t} \rightarrow \mathcal{X}$ with the domain and codomain equipped with the sup-norms. By the above-mentioned implicit function theorem, there exists a neighbourhood such that $\x_{t+1}$ is continuously differentiable with bounded derivatives. Since $\mathcal{X}$ is a compact subset of $\mathbb{R}^d$, the spaces $((\mathcal{X}\times \mathbb{R})^{t},\norm{\cdot}_{\infty})$ and $(\mathcal{X},\norm{\cdot}_{\infty})$ are Banach spaces. Thus, by the mean value inequality on Banach spaces \citep[][Theorem 12.6]{baggett1992functional}, if we consider the closed line segment joining $\mathcal{D}_t$ to $\mathcal{D}_t^{\text{SF}}$ (using the convexity of $\mathcal{X}$), and define $M_t$ to be the maximum rate of change over the line segment,
\begin{align}\label{M_t}
M_t &:= \max_{\mathcal{D} \in \mathbb{D}_t }\norm{\frac{\partial \x_{t+1}}{\partial \mathcal{D}}\left(\mathcal{D}\right)}_{\text{op}},~~\text{where}\ \mathbb{D}_t = \left\{ \mathcal{D}\ \mid\ \mathcal{D} = (1-s)\mathcal{D}_t + s\mathcal{D}_t^{\text{SF}},\ s \in [0,1] \right\}
\end{align}
we can bound the closeness of the next queries at round $t$,
\begin{align*}
 \norm{\x^{\text{SF}}_{t+1} - \x_{t+1}}_{\infty} = \norm{\x_{t+1}(\mathcal{D}_t^{\text{SF}})-\x_{t+1}(\mathcal{D}_t)}_{\infty} &\leq \norm{\mathcal{D}_t^{\text{SF}}-\mathcal{D}_t}_{\infty} M_t.
\end{align*}

It remains to bound $\norm{\mathcal{D}_t^{\text{SF}}-\mathcal{D}_t}_{\infty}$. First, observe that,
$$\norm{\mathcal{D}_t^{\text{SF}}-\mathcal{D}_t}_{\infty} = \max\left\{\norm{\x_{1}^{\text{SF}}-\x_{1}}_{\infty},...,\norm{\x_{t}^{\text{SF}}-\x_{t}}_{\infty}, \left|y_{1}^{\text{SF}}-y_{1}\right|,...,\left|y_{t}^{\text{SF}}-y_{t}\right|\right\}.$$
Let us consider $|y_{\tau}^{\text{SF}}-y_{\tau}|$ for any $\tau \in \{1,...,t\}$. It holds that,
\begin{align*}
  |y_{\tau}^{\text{SF}}-y_{\tau}| =
  \begin{cases}
    |f(\x_{\tau}^{\text{SF}})-f(\x_{\tau})|, & \text{ if Line 8 of Algorithm~\ref{mfbo-algo} false at query } \tau\\
    |f(\x_{\tau}^{\text{SF}})-\mu_{\tau}(\x_{\tau})|, & \text{ if Line 8 true at query } \tau
  \end{cases}
\end{align*}
where we use the shorthand notations $f(\x) := f(\x,m)$ and $\mu_{\tau,\text{MF}}(\x) := \mu_{\text{MF}}(\x,m \mid \mathcal{D}_{\tau})$.

For the false case (real observation), we have
\begin{align}
|y_{\tau}^{SF}-y_{\tau}| &= |f(\x_{\tau}^{\text{SF}})-f(\x_{\tau})| \nonumber \\
&\leq \sqrt{d} \norm{\x_{\tau}^{\text{SF}}-\x_{\tau}}_2 \label{kernel_as} \\
&\leq d \norm{\x_{\tau}^{\text{SF}}-\x_{\tau}}_{\infty} \label{equiv_norms}  \\
&\leq d \varepsilon \tau\hat{M}_{\tau} d^{\tau} \label{ih} \\
&= \varepsilon \tau\hat{M}_{\tau} d^{\tau +1} \nonumber
\end{align}
with probability greater than $1-da\exp(-\frac{1}{b^2})$ by Assumption \ref{smallderivative} exploited in \eqref{kernel_as}. The inequalities \eqref{equiv_norms} and \eqref{ih} follow from the equivalence of the norms ($\norm{\x}_2 \leq \sqrt{d}\norm{\x}_{\infty}$) and the induction hypothesis, respectively.
For the true case (pseudo-observations), as before we have
\begin{align*}
|y_{\tau}^{\text{SF}}-y_{\tau}| &\leq |f(\x_{\tau}^{\text{SF}})-\mu_{\tau,\text{MF}}(\x_{\tau})|\\ 
&\leq |f(\x_{\tau}^{SF})-f(\x_{\tau})| + |f(\x_{\tau})-\mu_{\tau,\text{MF}}(\x_{\tau})|\\
&\leq \varepsilon \tau\hat{M}_{\tau} d^{\tau +1} + |f(\x_{\tau})-\mu_{\tau,\text{MF}}(\x_{\tau})|
\end{align*}
with probability greater than $1-da\exp(-\frac{1}{b^2})$. The first term represents the error for being off from the single-fidelity algorithm acquisition track, and the second term is the prediction error of the MOGP surrogate model. Given that  the objective $f$ is drawn from a MOGP with same covariance kernel than that of the MOGP surrogate in the Algorithm (Assumption~\ref{f_is_GP}), the latter term can be bounded.
For any constant $C >0$, at round $t$,

\begin{align*}
    \prob\left(\frac{f(\x)-\mu_{t,\text{MF}}(\x)}{\sigma_{t,\text{MF}}(\x)} > C\right) &\leq \frac{1}{2}\exp\left(-\frac{C^2}{2}\right)\\
    \prob\left(|f(\x)-\mu_{t,\text{MF}}(\x)| > C\sigma_{t,\text{MF}}(\x)\right) &\leq \exp\left(-\frac{C^2}{2}\right)\\
    \prob\left(|f(\x)-\mu_{t,\text{MF}}(\x)| \leq C\sigma_{t,\text{MF}}(\x)\right) &\geq 1-\exp\left(-\frac{C^2}{2}\right).
\end{align*}
Pick $C=\frac{\varepsilon}{\sigma_{t,\text{MF}}(\x)}$. Then, we know that $|f(\x)-\mu_{t,\text{MF}}(\x)| \leq \varepsilon$ holds at least with probability $1-\exp\left(-\frac{\varepsilon^2}{2\sigma_{t,\text{MF}}^2(\x)}\right)$. 

If $\sigma_t(\x) \leq \frac{\varepsilon}{\sqrt{-2\log(1-q)}}$, then $|f(\x)-\mu_{t,\text{MF}}(\x)| \leq \varepsilon$ holds with probability greater than $q$. Therefore, $|f(\x_{\tau})-\mu_{\tau,\text{MF}}(\x_{\tau})| \leq \varepsilon$, and
\begin{align*}
|y_{\tau}^{\text{SF}}-y_{\tau}| &\leq \varepsilon \tau\hat{M}_{\tau} d^{\tau +1} + \varepsilon = \varepsilon(\tau\hat{M}_{\tau} d^{\tau +1} + 1).
\end{align*}

By combining the results, we have
\begin{align*}
\norm{\mathcal{D}_t^{\text{SF}}-\mathcal{D}_t}_{\infty} &= \max\left\{\norm{\x_{1}^{\text{SF}}-\x_{1}}_{\infty},\cdots,\norm{\x_{t}^{\text{SF}}-\x_{t}}_{\infty}, \left|y_{1}^{\text{SF}}-y_{1}\right|,\cdots,\left|y_{t}^{\text{SF}}-y_{t}\right|\right\}\\
&\leq \max\left\{\varepsilon M_0 d,\cdots,\varepsilon t d^t \hat{M}_{t}, \varepsilon(M_0 d +1 ),\cdots,\varepsilon(t d^{t+1} \hat{M}_{t} + 1)\right\}\\
&= \varepsilon\left( t d^{t+1} \hat{M}_{t} + 1\right)
\end{align*}
with probability greater than $q\left(1-da\exp(-\frac{1}{b^2})\right)$. Note that the event $|f(\x)-\mu_{t,\text{MF}}(\x)| \leq \varepsilon$ and the event in Assumption \ref{smallderivative} are independent given the assumptions.
For all $t \in \{1,...,T-2\}$, 

\begin{align*}
 \norm{\x^{\text{SF}}_{t+1} - \x_{t+1}}_{\infty} &\leq \varepsilon\left( t d^{t+1} \max_{S \in 2^{\llbracket t-1 \rrbracket}}\prod_{k \in S}M_k + 1\right)M_{t} \leq \varepsilon\left( tM_t d^{t+1} \max_{S \in 2^{\llbracket t-1 \rrbracket}}\prod_{k \in S}M_k + d^{t+1} \max_{S \in 2^{\llbracket t \rrbracket}}\prod_{k \in S}M_k\right)\\
 &= \varepsilon d^{t+1} \left( t\max_{S \in 2^{\llbracket t \rrbracket}}\prod_{k \in S}M_k + \max_{S \in 2^{\llbracket t \rrbracket}}\prod_{k \in S}M_k\right)\\
 &= \varepsilon(t+1)\hat{M}_{t+1} d^{t+1}
\end{align*}
holds with probability greater than $q\left(1-da\exp(-\frac{1}{b^2})\right)$.
\end{proof}

\subsection{Noisy scenario: Proposition \ref{bound_sequence}}\label{noisy_scenario}

We consider a noisy scenario, that is, $\sigma_{\text{noise}} > 0$. It can be shown that Proposition \ref{bound_sequence} holds if $\frac{\sqrt{\sigma_{\text{noise}}}}{\varepsilon} \leq (d^{t+1}\hat{M}_t - 1)/2$ for all $t$ (with a negligible lower probability, specifically a factor of $\textrm{erf}\left(\frac{1}{2\sqrt{\sigma_{\text{noise}}}}\right)\textrm{erf}\left(\frac{1}{\sqrt{2\sigma_{\text{noise}}}}\right)$ lower). Given empirical study on values $\hat{M}_t$ (Supplementary \ref{supp_computeM}), it is highly unlikely that this condition does not hold with reasonable values for $\varepsilon$ and $\sigma_{\text{noise}}$.

\begin{proof}\label{proof_noisy_scenario}

Proof \ref{proof_bound_sequence} should be modified as follows.

Let us consider $|y_{\tau}^{\text{SF}}-y_{\tau}|$ for any $\tau \in \{1,...,t\}$. It holds that,
\begin{align*}
  |y_{\tau}^{\text{SF}}-y_{\tau}| =
  \begin{cases}
    |f(\x_{\tau}^{\text{SF}})+\epsilon-f(\x_{\tau})-\epsilon'|, & \text{ if Line 8 of Algorithm~\ref{mfbo-algo} false at query } \tau\\
    |f(\x_{\tau}^{\text{SF}})+\epsilon-\mu_{\tau, \text{MF}}(\x_{\tau},m)|, & \text{ if Line 8 false at query } \tau.
  \end{cases}
\end{align*}

Note that $|\epsilon-\epsilon'|$ follows a half-normal distribution with scale parameter $\sqrt{2}\sigma_{\text{noise}}$, and $|\epsilon|$ follows a half-normal distribution with scale parameter $\sigma_{\text{noise}}$. This implies that $P(|\epsilon-\epsilon'| \leq \sqrt{\sigma_{\text{noise}}}) = \textrm{erf}\left(\frac{1}{2\sqrt{\sigma_{\text{noise}}}}\right)$ and $P(|\epsilon| \leq \sqrt{\sigma_{\text{noise}}}) = \textrm{erf}\left(\frac{1}{\sqrt{2\sigma_{\text{noise}}}}\right)$. 

For the false case (real observation), we have
\begin{align*}
|y_{\tau}^{\text{SF}}-y_{\tau}| &= |f(\x_{\tau}^{\text{SF}})+\epsilon-f(\x_{\tau})-\epsilon'|  \\
&\leq |f(\x_{\tau}^{\text{SF}})-f(\x_{\tau})| + |\epsilon-\epsilon'|  \\ 
&\leq \frac{1}{\sqrt{d}} \norm{\x_{\tau}^{\text{SF}}-\x_{\tau}} + \sqrt{\sigma_{\text{noise}}}\\
&\leq \frac{1}{\sqrt{d}} \sqrt{d} \varepsilon \tau\hat{M}_{\tau} + \sqrt{\sigma_{\text{noise}}}  \\
&= \varepsilon \tau\hat{M}_{\tau} + \sqrt{\sigma_{\text{noise}}} 
\end{align*}
with probability greater than $\textrm{erf}\left(\frac{1}{2\sqrt{\sigma_{\text{noise}}}}\right)\left(1-da\exp(-\frac{1}{b^2})\right)$ by Assumption \ref{smallderivative}. The last two inequalities follow from the induction hypothesis and the equivalence of the norms, $\norm{\x} \leq \sqrt{d}\norm{\x}_{\infty}$.

For the true case (pseudo-observation), as before we have,
\begin{align*}
|y_{\tau}^{SF}-y_{\tau}| &\leq |f(\x_{\tau}^{SF})-\mu(\x_{\tau})| + |\epsilon|\\ 
&\leq |f(\x_{\tau}^{SF})-f(\x_{\tau})| + |f(\x_{\tau})-\mu(\x_{\tau})| + \sqrt{\sigma_{\text{noise}}} \\
&\leq \varepsilon \tau\hat{M}_{\tau} + |f(\x_{\tau})-\mu(\x_{\tau})| + 2\sqrt{\sigma_{\text{noise}}} \\
&\leq \varepsilon \tau\hat{M}_{\tau} + \varepsilon + 2\sqrt{\sigma_{\text{noise}}} \\
&= \varepsilon(\tau\hat{M}_{\tau} + 1 + \frac{2\sqrt{\sigma_{\text{noise}}}}{\varepsilon})
\end{align*}
with probability greater than $\textrm{erf}\left(\frac{1}{2\sqrt{\sigma_{\text{noise}}}}\right)\textrm{erf}\left(\frac{1}{\sqrt{2\sigma_{\text{noise}}}}\right)\left(1-da\exp(-\frac{1}{b^2})\right)$. 

Hence, in this case we have
\begin{align*}
&\norm{\mathcal{D}_t^{\text{SF}}-\mathcal{D}_t}_{\infty} \\
&= \max\left\{\norm{\x_{1}^{\text{SF}}-\x_{1}}_{\infty},...,\norm{\x_{t}^{\text{SF}}-\x_{t}}_{\infty}, \left|y_{1}^{\text{SF}}-y_{1}\right|,...,\left|y_{t}^{\text{SF}}-y_{t}\right|\right\}\\
&\leq \max\left\{\varepsilon M_0 d,\cdots,\varepsilon t d^t \hat{M}_{t}, \varepsilon(M_0 d + 1+  \frac{2\sqrt{\sigma_{\text{noise}}}}{\varepsilon}),\cdots,\varepsilon\left(t d^{t+1} \hat{M}_{t} +  1 + \frac{2\sqrt{\sigma_{\text{noise}}}}{\varepsilon}\right)\right\}\\
&= \varepsilon\left( t d^{t+1} \hat{M}_{t} +  1 + \frac{2\sqrt{\sigma_{\text{noise}}}}{\varepsilon}\right)
\end{align*}
with probability greater than $\textrm{erf}\left(\frac{1}{2\sqrt{\sigma_{\text{noise}}}}\right)\textrm{erf}\left(\frac{1}{\sqrt{2\sigma_{\text{noise}}}}\right)\left(1-da\exp(-\frac{1}{b^2})\right)q$.

For all $t \in \{1,...,T-2\}$, 

\begin{align*}
 \norm{\x^{\text{SF}}_{t+1} - \x_{t+1}}_{\infty} &\leq \varepsilon\left( t d^{t+1} \max_{S \in 2^{\llbracket t-1 \rrbracket}}\prod_{k \in S}M_k + 1 + \frac{2\sqrt{\sigma_{\text{noise}}}}{\varepsilon}\right)M_{t} \\
 &\leq \varepsilon\left( tM_t d^{t+1} \max_{S \in 2^{\llbracket t-1 \rrbracket}}\prod_{k \in S}M_k + d^{t+1} \max_{S \in 2^{\llbracket t \rrbracket}}\prod_{k \in S}M_k\right)\\
 &= \varepsilon d^{t+1} \left( t\max_{S \in 2^{\llbracket t \rrbracket}}\prod_{k \in S}M_k + \max_{S \in 2^{\llbracket t \rrbracket}}\prod_{k \in S}M_k\right)\\
 &= \varepsilon(t+1)\hat{M}_{t+1} d^{t+1}
\end{align*}
holds with probability greater than $\textrm{erf}\left(\frac{1}{2\sqrt{\sigma_{\text{noise}}}}\right)\textrm{erf}\left(\frac{1}{\sqrt{2\sigma_{\text{noise}}}}\right)\left(1-da\exp(-\frac{1}{b^2})\right)q$, when $1 + \frac{2\sqrt{\sigma_{\text{noise}}}}{\varepsilon} \leq d^{t+1}\hat{M}_t$. Specifically, when $\frac{\sqrt{\sigma_{\text{noise}}}}{\varepsilon} \leq (d^{t+1}\hat{M}_t - 1)/2$.

\end{proof}

\subsection{Theorem \ref{no_harm_theorem}}

\begin{proof}[Proof of Theorem \ref{no_harm_theorem}]\label{proof_no_harm_theorem}
First, note that for any budget and any choice (simple or Bayes optimal) it holds,
\begin{align*}
    &R(\Lambda, \x_{\text{choice}}^{\text{SF}}) - R(\Lambda, \x_{\text{choice}}^{\text{rMF}})  = f(\x_{\text{choice}}^{\text{rMF}}) - f(\x_{\text{choice}}^{\text{SF}}).
\end{align*}

For the simple choice, we have $\x_{\text{choice}}^{\text{SF}} = \argmax_{t \in \llbracket T \rrbracket}f(\x_t^{\text{SF}})$, and $\x_{\text{choice}}^{\text{rMF}} = \argmax_{t \in \llbracket T^{(m)} \rrbracket}f(\x_t)$ where $\x_1,...,\x_{T^{(m)}}$ is the primary IS acquisition sequence returned by Algorithm \ref{mfbo-algo} (pseudo-queries removed from the output sequence). With a slight abuse of notation we write $T$ and $T^{(m)}$ for both the number of queries (Definition \ref{n_queries}) and the corresponding index sets (e.g.\ $t \in T^{(m)}$ means that $y_t$ is not a pseudo-observation).

For all $t \in T$, it holds that $f(\x_{\text{choice}}^{\text{rMF}}) - f(\x_{\text{choice}}^{\text{SF}}) > - \varepsilon T \hat{M}_T d^{T+1}$ by Corollary \ref{instant_regret_bounded} with probability greater than $q\left(1-da\exp(-\frac{1}{b^2})\right)$. The problem is that the values $y_t$ for $t \in T \setminus T^{(m)}$ are never observed, and we cannot take minimum over these “NaN values" (i.e., $\argmax$ is not well-defined) in the computation of $\x^{\text{rMF}}_{\text{choice}}$. To solve this issue, a quantity $\lambda_m$ was saved from the budget $\Lambda$ (Algorithm~\ref{mfbo-algo}, Lines 24-26), thus ensuring that if the true maximizer is one of the pseudo-observations, then it will be queried at primary IS, leading to an actual observation.

Specifically, for the last query at $T+1$. Note that every pseudo-query is in $S = \left\{ \x \in \mathcal{X}\ \mid\ \sigma_{T,\text{MF}}(\x,m) \leq c_1\right\}$. For any $\x \in S$, it holds that $P(|f(\x)-\mu_{T,\text{MF}}(\x)| \leq \varepsilon) \geq q$ (see Proof \ref{proof_bound_sequence}). Thus,
\begin{align*}
    \lvert\max_{\x \in S}f(\x) - f(\argmax_{x \in S}\mu_{T,\text{MF}}(\x))\rvert \leq 2\varepsilon,
\end{align*}
with probability greater than $q$.
\end{proof}

\section{Full version of the algorithm}\label{sec:fullversion}

Some modifications can be done to improve the empirical performance of Algorithm \ref{mfbo-algo} while all the theoretical results of Section \ref{theoretical_results} still hold.

\paragraph{Posterior mean update of the pseudo observations:}
Lines 12-13 of Algorithm \ref{mfbo-algo} are for simplicity, the algorithm can be made more efficient by adjusting these. All the pseudo-observations in $\mathcal{D}^{\text{pSF}}$ can be updated to correspond to the most recent predictive mean estimate of the joint surrogate model at the current round $t$. This does not break the condition $\sigma_{\text{MF}}(\x^{\text{pSF}}_t,m) \leq c_1$, since the posterior variance cannot increase as new data is added. We go further and also check whether the single-fidelity GP surrogate can provide a more accurate estimate of the pseudo-observation in the sense of the accuracy of a nearest neighbor. For the pseudo-observation, we choose the most recent predictive mean estimate of the single-fidelity surrogate model if $\lvert f(\x^{\text{NN}}_t,m) - \mu_{\text{SF}}(\x^{\text{pSF}}_t)\rvert \le \lvert f(\x^{\text{NN}}_t,m) - \mu_{\text{MF}}(\x^{\text{pSF}}_t, m)\rvert$ where $\x^{\text{nn}}_t$ is the nearest neighbor of $\x^{\text{pSF}}_t$ in the primary IS training data.

\paragraph{Multiple auxiliary IS relevance check:}
When the number of auxiliary ISs is more than two, the algorithm can give a chance also to other auxiliary IS, even if the first proposed query $\x^{\text{MF}}_t$ at IS $\ell_t$ is considered irrelevant by the algorithm. Looping over all IS, and checking their relevance, does not violate the conditions in Algorithm \ref{mfbo-algo}, so the theoretical results are preserved. 

\paragraph{Relevance check for primary IS:}
When $(\x^{\text{MF}}_t,\ell_t)$ with $\ell_t=m$ is proposed, we can either have or not have a relevance check for that primary IS query. We consider a version that does not have a relevance check, which means that if $\ell_t=m$, the query is automatically accepted. 

The pseudo code of the full algorithm is presented in Algorithm \ref{mfbo-algo-full}. Blue lines correspond to addition w.r.t.\ the first improvement, red lines to the second, and purple lines to the third.

\section{Computing constants $M_t$}\label{supp_computeM}

Recall that the formula for $M_t$ presented in Equations \eqref{dxdD} and \eqref{M_t}. The optimization over $\mathbb{D}_t$ makes the computation of $M_t$ expensive. To avoid this, we consider a lower bound for $M_t$, defined as,
\begin{align}\label{lower_M}
\ubar{M}_t := \norm{\left[\left(\mathbf{H}_{\alpha,\x}(\x_{t+1}(\mathcal{D}_t),\mathcal{D}_t) \right)^{-1}\frac{\partial \nabla_{\x}\alpha}{\partial\mathcal{D}_t^{(k)}}(\x_{t+1}(\mathcal{D}_t),\mathcal{D}_t)\right]_{k=1}^{t(d+1)}}_{\text{op}},
\end{align}
where $\norm{A}_{\text{op}} := \inf \{c>0 : \norm{A\x}_\infty \leq c\norm{\x}_\infty\ \forall \x \in \mathcal{X} \}$ is the operator norm, and $\mathbf{H}_{\alpha,\x}$ denotes the Hessian of $\x \mapsto \alpha(\x,\mathcal{D}_t)$. Specifically, the $(i,j)^{\text{th}}$-element of $\mathbf{H}_{\alpha,\x}$ reads as $\frac{\partial^2 \alpha}{\partial x_i \partial x_j}$. The $i^{\text{th}}$-element of $\frac{\partial\nabla_{\x} \alpha}{\partial\mathcal{D}^{(k)}_t}$ is $\frac{\partial^2 \alpha}{\partial\mathcal{D}^{(k)}_t\partial x_i}$, which denotes the partial derivative w.r.t.\ the first and the second variable of $\alpha(\cdot,\cdot)$. Note that $\mathbf{H}_{\alpha,\x}(\mathbf{a},\mathbf{b}) \in \mathbb{R}^{d\times d}$ and $\frac{\partial\nabla_{\x} \alpha}{\partial\mathcal{D}^{(k)}_t}(\mathbf{a},\mathbf{b}) \in \mathbb{R}^d$ for $\mathbf{a} \in \mathbb{R}^{d}, \mathbf{b} \in \mathbb{R}^{t(d+1)}$.

The gradient of $\nabla \alpha_{\x}$ can be obtained by exploiting the automatic differentiation tools available in different programming frameworks. We used the BoTorch-GPyTorch ecosystems \citep{balandat2020botorch,gardner2018gpytorch}. Matrices in Equation~\eqref{lower_M} need not to compute separately, but instead by taking the Jacobian of $(\x_{t+1},\mathcal{D}_t) \mapsto \nabla_{\x}\alpha(\x_{t+1}, \mathcal{D}_t)$, and by considering its sub-matrices, all the terms can be obtained. However, the automatic differentiation comes at the cost of possible numerical instability. Especially, a reliable estimate of the Hessian $\mathbf{H}_{\alpha,\x}$ turned out to be difficult to obtain, resulting often a Hessian with complex eigen values and lacking symmetry. However, we run an experiment where the Hessian was forced to be symmetric and a large jitter term was added to the diagonal. The results on Rosenbrock 2D with rMF-GIBBON over 20 repetitions are depicted in Table \ref{table_mvalues}.

\begin{table}[t]
\centering
\begin{tabular}{|c|c|c|c|c|c|c|c|c|c|}
	\hline
	\textbf{Round} & $1$ & $2$ & $3$ & $4$ & $5$ & $6$ & $7$ & $8$ & $9$ \\
	\hline\hline
	Mean & 0.823668 & 1.081626 & 1.005563 & 1.199914 & 1.483624 &  1.454633 & 1.854412 & 2.625723 & 3.606273\\
	\hline
	Std & 0.398510 & 0.475530 & 0.358448 & 0.610190 & 0.559461 & 0.511228 & 0.903042 & 1.302165 & 1.784013\\
	\hline
\end{tabular}
\caption{The mean and standard deviation of $\ubar{M}_t, t \in \llbracket9\rrbracket$, over 20 repetitions.} 
\label{table_mvalues}
\end{table}

$\ubar{M}_t$ grows as more data is obtained (as $t$ grows), as expected. Namely, $\ubar{M}_t$ is same as (the maximum over $j$) the sum over $(|\partial x_j / \partial \mathcal{D}_{t,1}|,...,|\partial x_j / \partial \mathcal{D}_{t,t}|)$, where $x_j$ denotes $j^{th}$ coordinate of $\x_{t+1}$ and $\mathcal{D}_{t,i}$ denotes $i^{th}$ data point of $\mathcal{D}_{t}$. That is, as the number of data points grows, the number of terms in the sum grows also.

\begin{minipage}[t]{9.8cm}
\begin{algorithm}[H]
\caption{Full version of robust MFBO algorithm}
\label{mfbo-algo-full}
    \begin{algorithmic}
        \STATE \textbf{Input}: Budget $\Lambda$, costs $(\lambda_1,...,\lambda_m)$, acquisition function $\alpha$, hyperparameters $c_1$ and $c_2$, relevance measure $s$
        \STATE Initialize $\mathcal{D}^{\text{pSF}},\mathcal{D}^{\text{MF}}$
        \STATE Perform Bayesian updates $\mu_{\text{SF}},\sigma_{\text{SF}},\mu_{\text{MF}},\sigma_{\text{MF}}$
        \color{blue}
        \STATE $p_{\text{obs}} \gets \{\}$
        \color{black}
        \STATE $t \gets 1$
        \WHILE {$\lfloor \Lambda / \lambda_m \rfloor \geq 2 \lambda_m$}
            \STATE $\x^{\text{pSF}}_t \gets \argmax_{\x} \alpha(\x,m \mid  \mu_{\text{SF}},\sigma_{\text{SF}})$
            \STATE condition1 $\gets \sigma_{\text{MF}}(\x^{\text{pSF}}_t,m) \leq c_1$
            \color{red}
            \STATE condition2 $\gets$ False
            \IF{condition1}
            \color{black}
            \STATE $(\x^{\text{MF}}_t,\ell_t) \gets \argmax_{\x,\ell} \alpha(\x,\ell  \mid  \mu_{\text{MF}},\sigma_{\text{MF}})$
            \color{purple}
            \IF{$\ell_t=m$}
            \STATE condition2 $\gets$ True
            \ELSE
            \color{red}
            \STATE $\textrm{ISleft} \gets \llbracket m -1 \rrbracket$
            \WHILE{$|\textrm{ISleft}|>0$ \AND \NOT condition2}
            \STATE $(\x^{\text{MF}}_t,\ell_t) \gets \argmax_{\x \in \mathcal{X}, \ell \in \textrm{ISleft}} \alpha(\x,\ell  \mid  \mu_{\text{MF}},\sigma_{\text{MF}})$
            \IF{$s(\x^{\text{MF}}_t,\ell_t) \geq c_2$}
            \STATE condition2 $\gets$ True
            \ELSE
            \STATE $\textrm{ISleft} \gets \textrm{ISleft} \setminus \{\ell_t\}$
            \ENDIF
            \ENDWHILE
            \ENDIF
            \ENDIF
            \color{black}
            \IF{condition1 \AND condition2}
                \color{blue}
                \STATE $p_{\text{obs}} \gets p_{\text{obs}} \cup \{t\}$
                \color{black}
                \STATE $y_t \gets f(\x^{\text{MF}}_t,\ell_t)$
                \STATE $\mathcal{D}^{\text{MF}} \gets \mathcal{D}^{\text{MF}} \cup \{((\x^{\text{MF}}_t,\ell_t),y_t)\}$
                \STATE Perform Bayesian updates $\mu_{\text{MF}},\sigma_{\text{MF}}$
                \STATE $y_t \gets \mu_{\text{MF}}(\x^{\text{pSF}}_t,m)$  \ \# pseudo-observation
                \STATE $\mathcal{D}^{\text{pSF}} \gets \mathcal{D}^{\text{pSF}} \cup \{(\x^{\text{pSF}}_t,y_t)\}$
                \STATE $\Lambda \gets \Lambda - \lambda_{\ell_t}$
            \ELSE
                \STATE $y_t \gets f(\x^{\text{pSF}}_t,m)$
                \STATE $\mathcal{D}^{\text{pSF}} \gets \mathcal{D}^{\text{pSF}} \cup \{(\x^{\text{pSF}}_t,y_t)\}$
                \STATE $\mathcal{D}^{\text{MF}} \gets \mathcal{D}^{\text{MF}} \cup \{((\x^{\text{pSF}}_t,m),y_t)\}$
                \STATE $\Lambda \gets \Lambda - \lambda_m$
            \ENDIF
            \STATE Perform Bayesian updates $\mu_{\text{SF}},\sigma_{\text{SF}},\mu_{\text{MF}},\sigma_{\text{MF}}$
            \color{blue}
            \STATE \textcolor{blue}{$\mathcal{D}^{\text{pSF}} \gets \textrm{UPDATE-PSEUDO-OBS}(\mathcal{D}^{\text{pSF}},\mathcal{D}^{\text{MF}},\mu_{\text{MF}},\mu_{\text{pSF}},p_{\text{obs}})$}
            \STATE \textcolor{blue}{Perform Bayesian updates $\mu_{\text{SF}},\sigma_{\text{SF}}$}
            \color{black}
            \STATE $t \gets t + 1$
        \ENDWHILE
        \STATE $S \gets \left\{ \x \in \mathcal{X}\ \mid\ \sigma_{\text{MF}}(\x,m) \leq c_1\right\}$
        \STATE $\x^{\text{pSF}}_t \gets \argmax_{\x \in S} \mu_{\text{MF}}(\x,m)$
        \STATE $y_t \gets f(\x^{\text{pSF}}_t,m)$
        %\STATE \textbf{return} $\left((\x_1,y_1),...,(\x_t,y_t)\right)$
    \end{algorithmic}
\end{algorithm}
\end{minipage}
\hfil
\begin{minipage}[t]{7.2cm}
\begin{algorithm}[H]
\caption{UPDATE-PSEUDO-OBS}
    \begin{algorithmic}
    \color{blue}
        \STATE \textbf{Input}: $\mathcal{D}^{\text{pSF}},\mathcal{D}^{\text{MF}},\mu_{\text{MF}},\mu_{\text{pSF}},pobs$ 
        \FOR{t in pobs}
        %\IF{$\sigma_{\text{SF}}(\x^{\text{pSF}}_t) \leq c_1$}
            \STATE $\x^{nn}_t \gets \text{NearestNeighbor}(\x^{\text{pSF}}_t,\mathcal{D}^{\text{MF}}[\ell=m])$
            \STATE $y \gets f(\x^{nn}_t,m)$
            \IF{$|\mu_{\text{MF}}(\x,m)-y| >
            |\mu_{\text{SF}}(\x)-y|$}
            \STATE $\mathcal{D}^{\text{pSF}}[y_t] \gets \mu_{\text{SF}}(\x^{\text{pSF}}_t)$
            %\ELSE
            %\STATE $\mathcal{D}^{\text{pSF}}[y_t] \gets \mu_{\text{MF}}(\x^{\text{pSF}}_t,m)$
            %\ENDIF
        \ELSE
            \STATE $\mathcal{D}^{\text{pSF}}[y_t] \gets \mu_{\text{MF}}(\x^{\text{pSF}}_t,m)$
        \ENDIF
        \ENDFOR
        \RETURN $\mathcal{D}^{\text{pSF}}$
    \color{black}
    \end{algorithmic}
\end{algorithm}  
\end{minipage}

\clearpage

\section{Hyperparameter $c_2$}\label{supp_c2}

\subsection{Non-adaptive strategy}

In experiments, we use a constant value $c_2=0.1$ over all BO rounds. To understand how $c_2$ is connected to the maximum information gain of the primary IS, we consider the maximum entropy search formula.
\cite{wang2017max} gives formula for the information gain with single $y^*$ draw (Formula 6, $K=1$), 
\begin{align*}
\textrm{I}(\{\x,y\}; y_*\mid D_t) &\approx \frac{\gamma_{y_*}( \x)\psi(\gamma_{y_*}( \x))}{2\Psi(\gamma_{y_*}(\x))} - \log(\Psi(\gamma_{y_*}( \x))),
\end{align*}
where $\psi$ is the probability density function and $\Psi$ the cumulative density function of a normal distribution, and $\gamma_{y_*}(\x) = \frac{y_* - \mu_t(\x)}{\sigma_t(\x)}$. The information gain $\textrm{I}$ is unbounded above but rarely in practice greater than $-\log(1/2)$, which is achieved when $\gamma_{y_*}(\x) = 0$. Then, roughly speaking, $c_2 = 0.1$ implies that the AIS query should give at least about $15\%$ of the maximum information gain. We recommend setting $c_2 = -u\log(1/2)$, where $u$ is the percent of the maximum information gain required for a cost-adjusted AIS query. We found that $u=15\%$ was a good default value.

\subsection{Adaptive strategy}\label{sec:adap}

As entropy decreases during BO rounds, the information gain also decreases. For this reason, we also consider a strategy for setting $c_2$, which adjusts to the amount of entropy at round $t$. Furthermore, this strategy automatically sets $c_2$ without the user having to specify it. We consider adaptively set $c_2(t) = a \int \text{I}(f(\x,m), f_* \mid \mathcal{D}^{\text{SF}}_t)d\x$, where $a$ is a positive constant. When $a = 1/\textrm{vol}(\mathcal{X})$, the threshold corresponds to the average information gain coming from the primary IS. However, this is too soft threshold value, as there are often large areas of space where the information gain is negligible, which in turn lowers the average information gain. For this reason, we found that a higher threshold works better empirically. Specifically, we found that the value $a = 100/\textrm{vol}(\mathcal{X})$ works well.

The information gain $\text{I}(f(\x,m), f_* \mid \mathcal{D}^{\text{SF}}_t)$ is computed by using the GP model that is trained on PIS data only. Pseudo-observations are not considered, so that they do not distort the information gain estimate. This requires one more GP model to be trained in the robust MFBO algorithm.
Figure~\ref{fig:bigmatrixadap} reproduces the results previously displayed in Figure~\ref{fig:bigmatrix} using now the proposed adaptive criteria for $c_2$.

\begin{figure}[t]
\begin{center}
		\includegraphics[width=\textwidth]{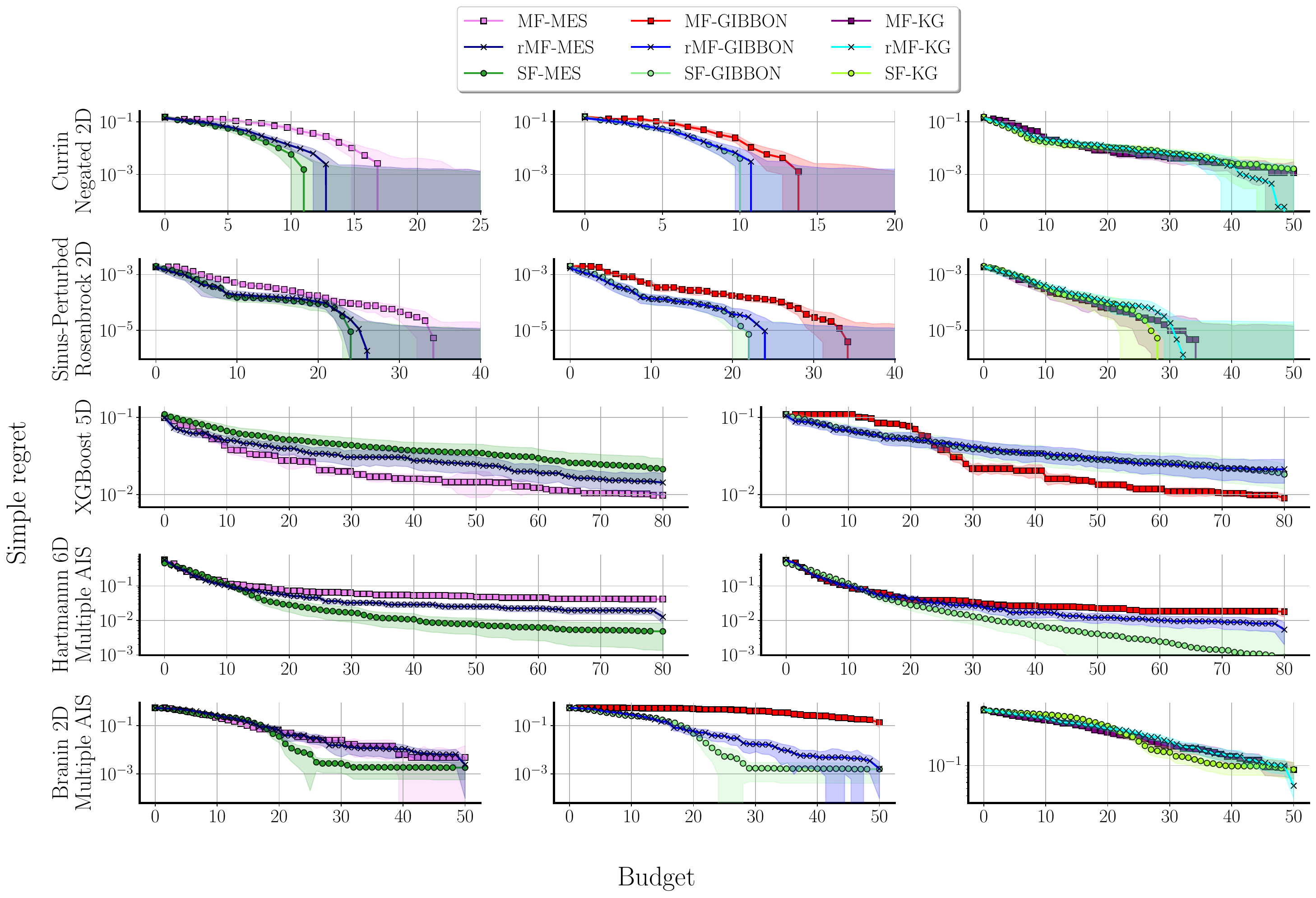}
	\end{center}
	\caption{Simple regret depicted over budget spent in five multi-fidelity problems, averaged over 100 repetitions. The first three problems have one auxiliary IS and the last two have three. Three BO algorithmic families (MES,GIBBON,KG) are tested with their multi- and single-fidelity variants. The proposed \textit{robust} multi-fidelity method is denoted by the letter `r', e.g.\ `rMF-MES'. For rMF methods, the selection of $c_2$ is done using the adaptive strategy proposed in Supplementary Section~\ref{sec:adap}.}
\label{fig:bigmatrixadap}
\end{figure}

\clearpage

\section{Multi-fidelity kernels}\label{supp_mf-kernels}

We here give some insights about the different joint models that can be used in MFBO, as well as some additional numerical experiments using different kernels.

\subsection{Kernels}

\paragraph{The Downsampling kernel \citep[][Supplementary]{wu2019practical}:} Recall that the joint model employed in the experiments from the main text uses the following kernel:
\begin{align*}
    k_{\text{DS}}((\x,\ell),(\x',\ell')) &= k_{\text{input}}(\x, \x') \times k_{\text{IS}}(\ell,\ell')\\
    k_{\text{input}}(\x, \x') &= \exp\left(-\frac{1}{2}\sum_{i=1}^d\frac{(x_i-x_i')^2}{s_i}\right)\\
    k_{\text{IS}}(\ell,\ell') &= c + (1-\ell)^{1+\delta}(1-\ell')^{1+\delta}
\end{align*}
The value $\ell \in [0,1]$ needs to be specified, and represents the confidence we have in the IS, with the primary IS $m$ being associated to $\ell=1$. Figure~\ref{fig:hartmannconfidence} investigates the effect of $\ell$. The hyperparameters $c, \delta$ and $\{s_i\}_{1\le i \le d}$ are obtained through marginal likelihood maximization.
When $\delta=0$,
\begin{align}
    k_{\text{DS}}((\x,\ell),(\x',\ell')) = (c + (1-\ell)(1-\ell'))k_{\text{input}}(\x, \x')
\label{eq:downsamp}
\end{align}
which can be written as
\begin{align*}
    k_{\text{DS}}((\x,\ell),(\x',\ell')) &=  \begin{cases} 
       ck_{\text{input}}(\x, \x') +(1-\ell)(1-\ell')k_{\text{input}}(\x, \x') & \ell \neq 1,~\ell' \neq 1 \\
      ck_{\text{input}}(\x, \x') & \text{otherwise}
   \end{cases}
\end{align*}

\begin{figure}[t]
    \centering
	\includegraphics[width=0.925\linewidth]{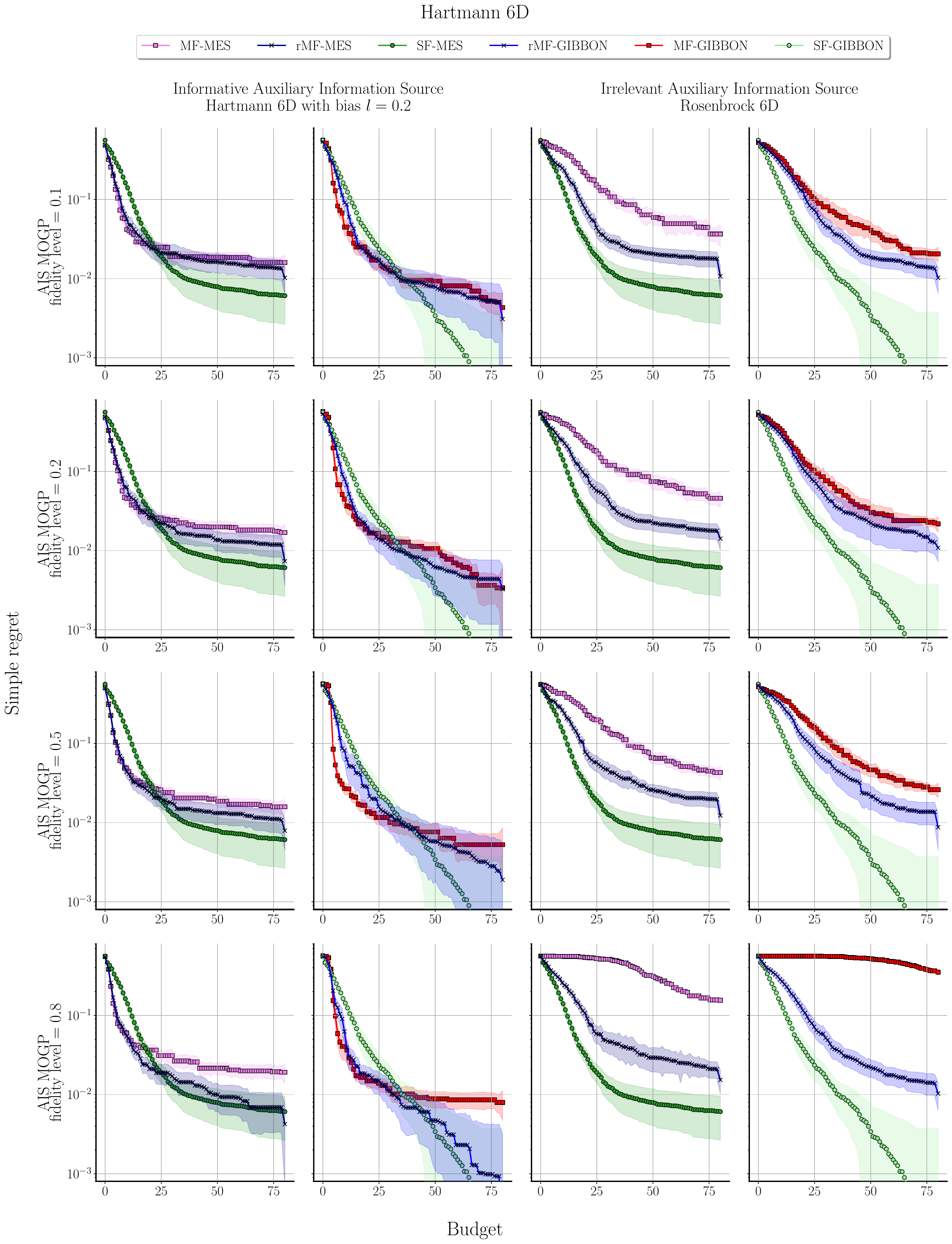}
	\caption{Simple regret depicted over budget spent for the Hartmann6D multi-fidelity problem, averaged over 100 repetitions. For the informative auxiliary IS, the Hartmann function with bias $l=0.2$ is considered (see Section~\ref{sec:testfunctions}). The irrelevant auxiliary IS is the 6D Rosenbrock function. For each row, the confidence level $\ell$ that the MOGP places in the auxiliary IS is varied. The MOGP considered uses the downsampling kernel.}
\label{fig:hartmannconfidence}
\end{figure}

\paragraph{The Linear Truncated kernel:}\label{par:lt}
The Linear Truncated kernel implemented in \href{https://github.com/pytorch/botorch}{BoTorch} reads as
\begin{align}
k_{\text{LT}}((\x,\ell),(\x',\ell')) &= k_{\text{input}}(\x,\x') + c(\ell, \ell')k_{\text{IS}}(\x,\x') \label{eq:lt}\\
c(\ell, \ell')& = (1 - \ell)(1 - \ell')(1 + \ell \ell')^p
\end{align}
where $k_{\text{input}}$ and $k_{\text{IS}}$ are Matern kernels both with $\nu=2.5$, but each with their own lengthscale.
For $p=0$, this leads to
\begin{align*}
k_{\text{LT}}((\x,\ell),(\x',\ell')) &=  \begin{cases} 
      k_{\text{input}}(\x, \x') + (1-\ell)(1-\ell') k_{\text{IS}}(\x, \x') & \ell \neq 1,~\ell' \neq 1\\
      k_{\text{input}}(\x, \x') & \text{otherwise}
   \end{cases}
\end{align*}

This highlights a close correspondence with the Downsampling kernel when the hyperparameters of $k_{\text{IS}}$ are close to that of $k_{\text{input}}$.
Figure~\ref{fig:bigmatrixlt} reproduces the results previously displayed in Figure~\ref{fig:bigmatrix} using now the linear truncated kernel as multiple output gaussian process kernel

\begin{figure}[t]
    \begin{center}
		\includegraphics[width=\textwidth]{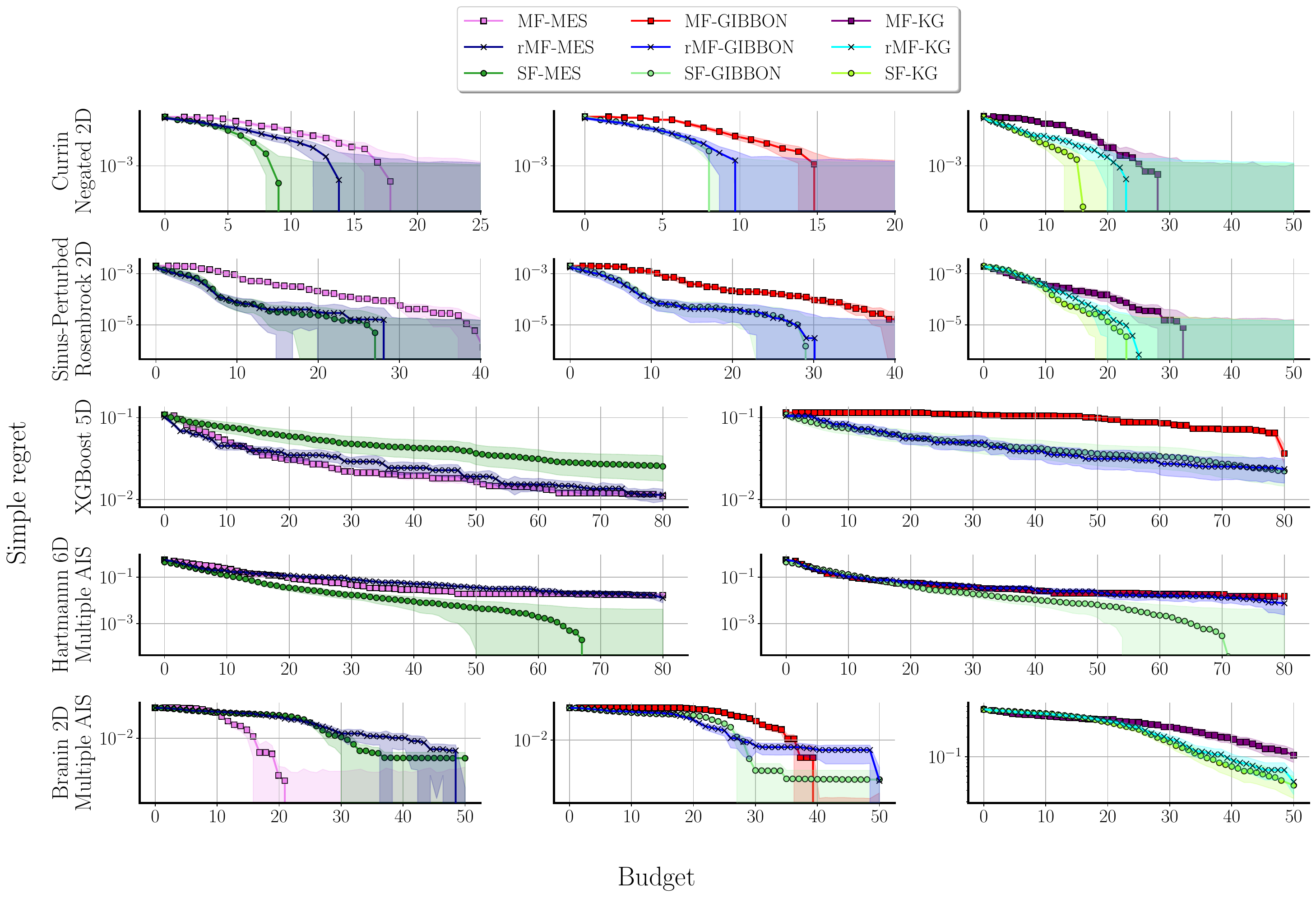}
	\end{center}
	\caption{Simple regret depicted over budget spent in five multi-fidelity problems, averaged over 100 repetitions. The first three problems have one auxiliary IS and the last two have three. Three BO algorithmic families (MES,GIBBON,KG) are tested with their multi- and single-fidelity variants. The proposed \textit{robust} multi-fidelity method is denoted by the letter `r', e.g.\ `rMF-MES'. For MF methods, the joint model kernel used is the linear truncated kernel described in Paragraph~\ref{par:lt}.}
\label{fig:bigmatrixlt}
\end{figure}

\paragraph{The MISO kernel \citep[][p.3]{poloczek2017multi}:}\label{par:miso}
Following our notations, the Multi-Information Source Optimization (MISO) kernel reads as
\begin{equation*}
k_{\text{MISO}}((\x,\ell),(\x',\ell')) = k_\text{input}(\x, \x') + \mathbb{I}(\ell=\ell') k_{\ell}(\x,\x')
\label{eq:misokernel_wrong}
\end{equation*}
where $k_\text{input}$ and $k_\ell$ are similar kernels, e.g.\ both Matern or RBF, but each with their own lengthscale. We assume that there is a typo in the text, and the correct formula should be,
\begin{equation}
k_{\text{MISO}}((\x,\ell),(\x',\ell')) = k_\text{input}(\x, \x') + \mathbb{I}(\ell=\ell'\neq 1) k_{\ell}(\x,\x')
\label{eq:misokernel}.
\end{equation}
Here, $\ell$ and $\ell'$ take categorical values, corresponding to ISs indexes, with $m$ being the primary IS index, equivalent to $\ell=1$ for the Downsampling and Linear Truncated kernels. This can also be written as
\begin{align*}
k_{\text{MISO}}((\x,\ell),(\x',\ell')) &=  \begin{cases} 
      k_\text{input}(\x, \x') + k_\ell(\x, \x') & \ell=\ell'\neq 1 \\
      k_{\text{input}}(\x, \x') & \text{otherwise}
   \end{cases}
\end{align*}

Figure~\ref{fig:bigmatrixmiso} reproduces the results previously displayed in Figure~\ref{fig:bigmatrix} using now the MISO kernel as multiple output gaussian process kernel

\begin{figure}[t]
    \begin{center}
	   \includegraphics[width=\textwidth]{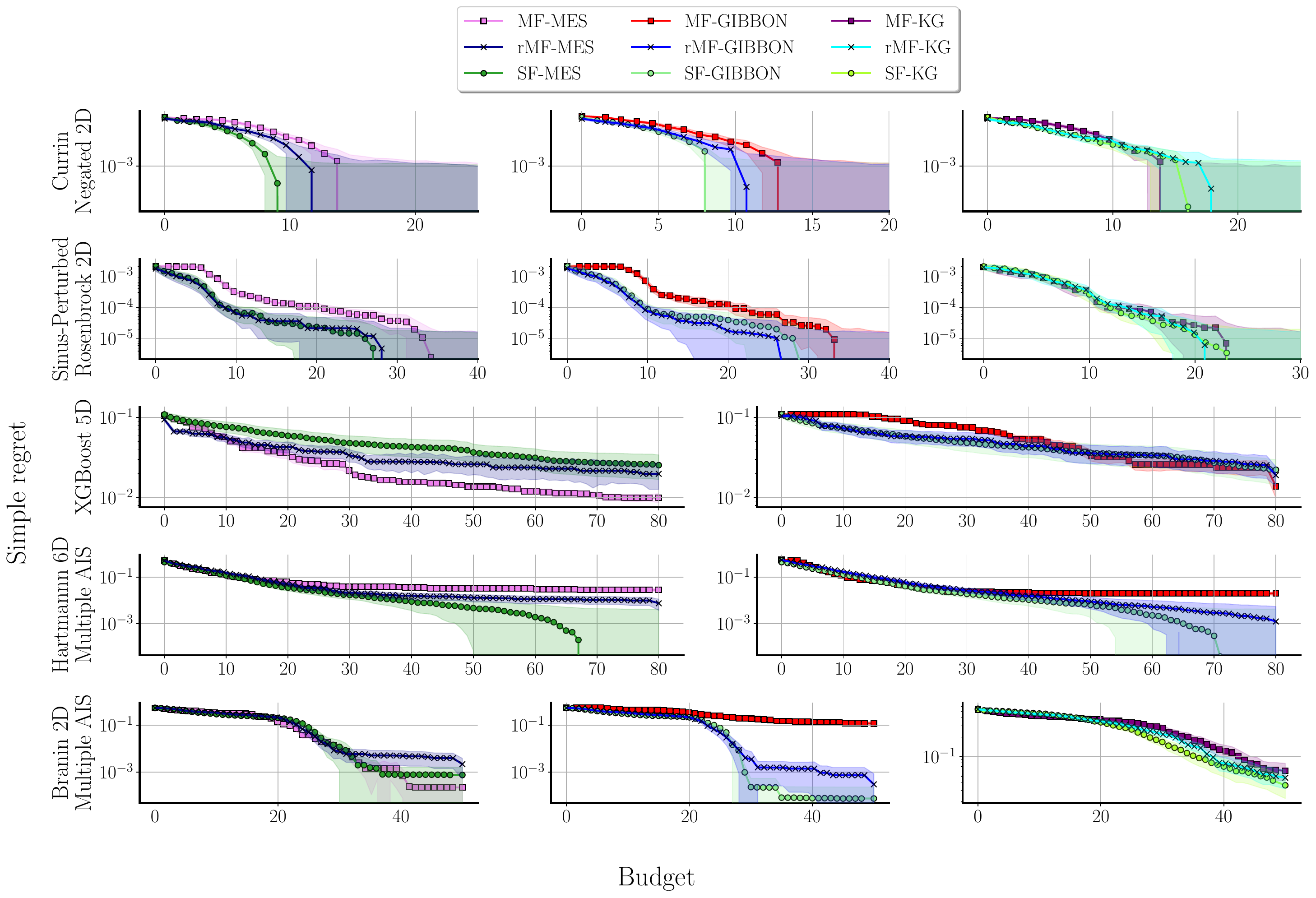}
	\end{center}
	\caption{Simple regret depicted over budget spent in five multi-fidelity problems, averaged over 100 repetitions. The first three problems have one auxiliary IS and the last two have three. Three BO algorithmic families (MES,GIBBON,KG) are tested with their multi- and single-fidelity variants. The proposed \textit{robust} multi-fidelity method is denoted by the letter `r', e.g.\ `rMF-MES'. For MF methods, the joint model kernel used is the MISO kernel described in Paragraph~\ref{par:miso}.}
\label{fig:bigmatrixmiso}
\end{figure}

\clearpage

\section{Experiment details}\label{sec:testfunctions}

\paragraph{Hartmann-6D function:}
\begin{align*}
    f(\x, \l) &= - \sum_{i=1}^4 \alpha_i \exp \left(- \sum_{j=1}^6 A_{ij}(x_j - P_{ij}) \right)\\
    \alpha &= (1.0 - 0.1 (1-l),1.2,3.0,3.2)^T\\
    \bf{A}   &= \begin{pmatrix}
    10&3&17&3.5&1.7&8\\
    0.05&10&17&0.1&8&14\\
    3&3.5&1.7&10&17&8\\
    17&8&0.05&10&0.1&14
    \end{pmatrix}\\
    \bf{P} &= 10^{-4} \begin{pmatrix}
    1312&1696&5569&124&8283&5886\\
    2329&4135&8307&3736&1004&9991\\
    2348&1451&3522&2883&3047&6650\\
    4047&8828&8732&5743&1091&381
    \end{pmatrix}
\end{align*}

defined over $[0,1]^6$, and $l\in[0,1]$ is the degree of fidelity. 
The primary IS is then reached for $l=1$.

\paragraph{Rosenbrock-$d$D function:}
\begin{equation*}
    f(\x) = \sum_{i=1}^{d-1}\left(100(x_{i+1}-x_i^2)^2 + (x_i-1)^2\right)
\end{equation*}

defined over $[-5,5]^d$.
The sinus-perturbed version used in the 2D case is defined as:

\begin{equation*}
    g(\x) = f(\x) + \mathbb{E}[f(X)] \times 0.8 \sin\left(x_1+x_2\right)
\end{equation*}

The expectation is approximated by the empirical mean taken over a grid of $1000\times1000$ points linearly spaced across $[-5,5]^2$.

\paragraph{Exponential Currin 2D function:}

\begin{equation*}
    f(\x) = \left ( 1-\exp\left(-\frac{1}{2x_2}\right)\right) \frac{2300x_1^3 + 1900x_1^2+2092x_1+60}{100x_1^3+500x_1^2+4x_1+20}
\end{equation*}

defined over $[0, 1]^2$.

\paragraph{Branin 2D function:}
\begin{equation*}
    f(\x, l) = \left(x_2 - \left(\frac{5.1}{4\pi^2} - 0.1 (1-l)\right) x_1^2 + \frac{5}{\pi} x_1 - 6\right)^2 + 10 \left(1 - \frac{1}{8\pi}\right) \cos(x_1) + 10
\end{equation*}

defined over $[-5,10]\times[0,15]$, and $l\in[0,1]$ is the degree of fidelity.

\paragraph{Ackley 2D function:}
\begin{equation*}
    f(\x) = -20 \exp \left(-0.2 \sqrt{\frac{1}{2}(x_1 + x_2)^2}\right) - \exp\left(\frac{1}{2}(\cos(2\pi x_1)+\cos(2 \pi x_2) \right) + 20 + e^1
\end{equation*}

defined over $[-5,10]\times[0,15]$.

\paragraph{XGBoost hyperparameter tuning:}

The following hyperparameters are optimized:
Huber loss parameter $\alpha \in [0.01, 0.1]$, complexity parameter used for minimal cost-complexity pruning $([0.01, 100])$, fraction of samples used to fit individual base learners $([0.1,1])$,
fraction of features considered when looking for the best tree split $([0.01, 1])$ and learning rate $([0.001, 1])$.
For the simple regret computation, $f^*$ has been obtained using 30000 evaluations of the primary IS at random points.

\begin{figure}[t]
    \centering
    \includegraphics[width=.75\textwidth]{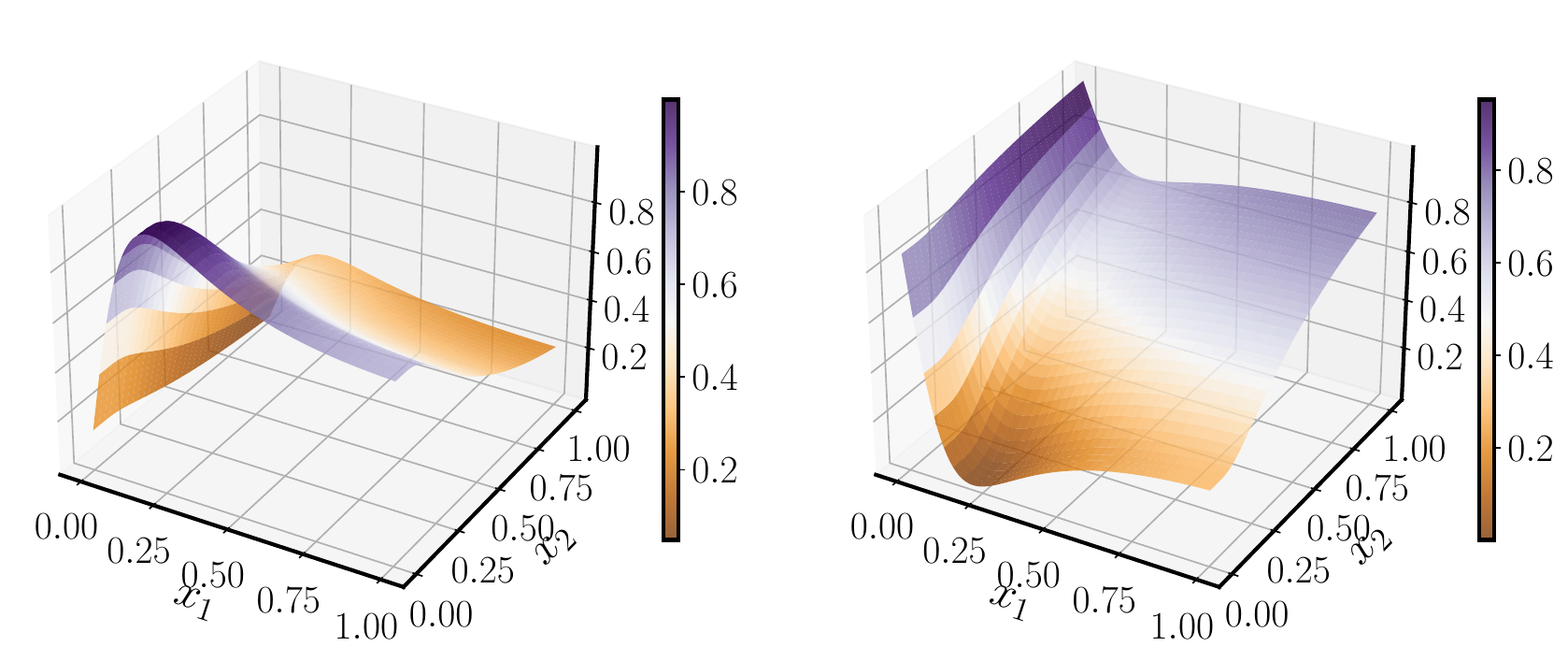}
    \includegraphics[width=.75\textwidth]{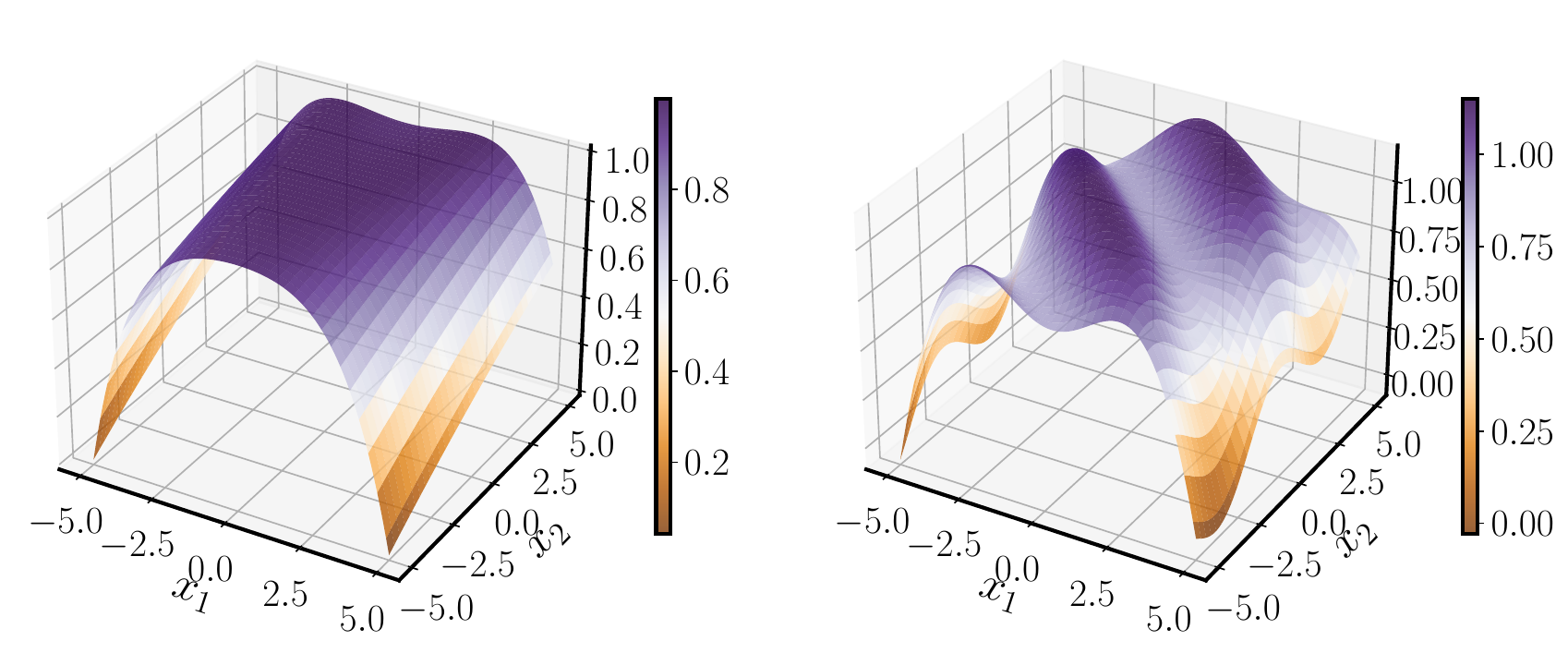}
    \includegraphics[width=.9\textwidth]{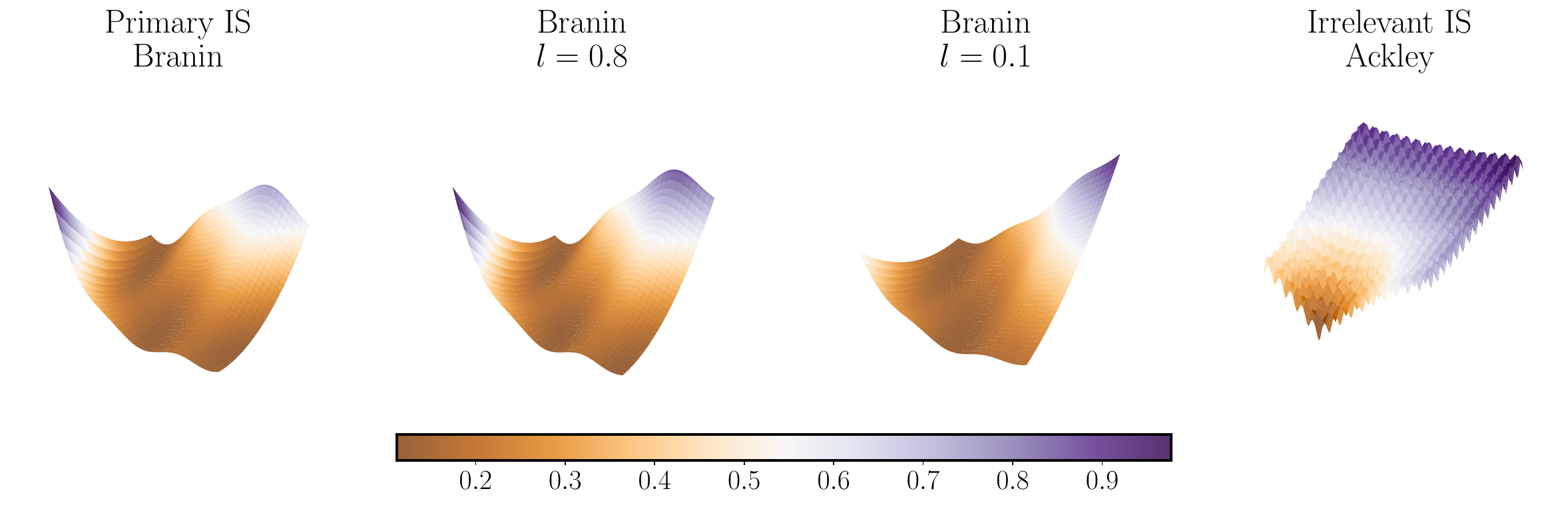}
	\caption{\textbf{Top:} Currin Negated 2D plot of the primary IS (left) and the auxiliary IS (right).
	\textbf{Middle:} Sinus-perturbed Rosenbrock 2D plot of the primary IS (left) and the auxiliary IS (right).
	\textbf{Bottom:} 2D Plot of the Multiple IS Branin problem.}
\label{fig:plots}
\end{figure}

\begin{figure}[t]
    \centering
	\includegraphics[width=\textwidth]{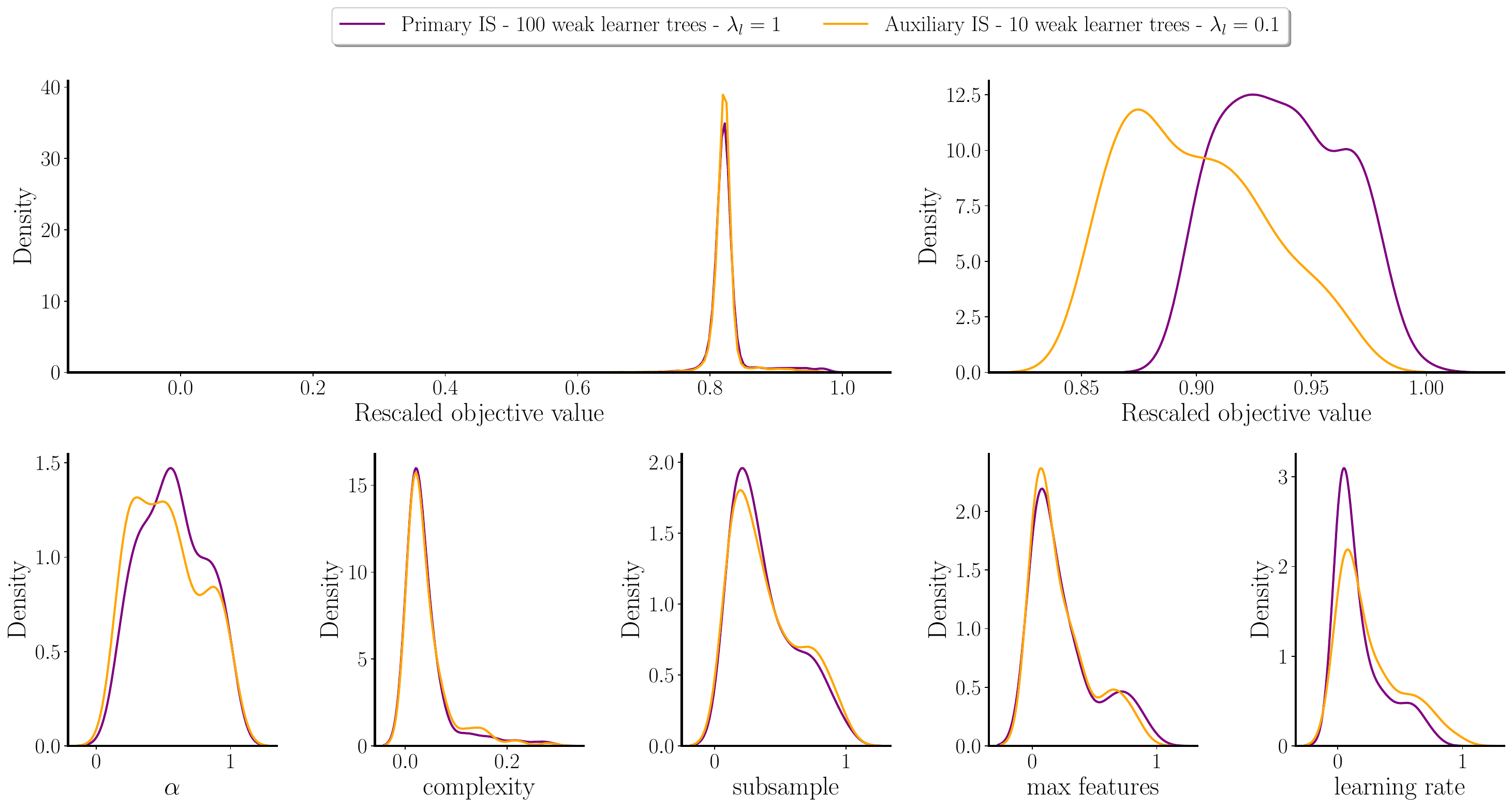}
	\caption{Distribution of rescaled objective values and best samples for the XGBoost 5D hyperparameter tuning problem. For each IS, distributions are computed using 3000 random uniform samples within hyperparameter bounds and a kernel density estimator. \textbf{Top left:} the whole rescaled objective values distribution. \textbf{Top right:} density plot of the 5\% best values. \textbf{Bottom:} density plot of the hyperparameters samples associated with 5\% best values, demonstrating the strong agreement between the primary and the auxiliary IS.}\label{fig:xgbdist}
\end{figure}

\end{document}